\documentclass[letterpaper]{article} % DO NOT CHANGE THIS
\usepackage{aaai23}  % DO NOT CHANGE THIS
\usepackage{times}  % DO NOT CHANGE THIS
\usepackage{helvet}  % DO NOT CHANGE THIS
\usepackage{courier}  % DO NOT CHANGE THIS
\usepackage[hyphens]{url}  % DO NOT CHANGE THIS
\usepackage{graphicx} % DO NOT CHANGE THIS

\usepackage{natbib}  % DO NOT CHANGE THIS AND DO NOT ADD ANY OPTIONS TO IT
\usepackage{caption} % DO NOT CHANGE THIS AND DO NOT ADD ANY OPTIONS TO IT
\usepackage[ruled,vlined,linesnumbered]{algorithm2e}
\SetKwRepeat{Do}{do}{while}

\SetCommentSty{mycommfont}

\usepackage{amsmath}
\usepackage{amssymb}
\usepackage{mathtools}
\usepackage{booktabs}
\usepackage{multirow}
\usepackage{tikz}
\usetikzlibrary{backgrounds}
\usetikzlibrary{arrows}
\usetikzlibrary{shapes,shapes.geometric,shapes.misc}

\tikzstyle{tikzfig}=[baseline=-0.25em,scale=0.5]

\pgfkeys{/tikz/tikzit fill/.initial=0}
\pgfkeys{/tikz/tikzit draw/.initial=0}
\pgfkeys{/tikz/tikzit shape/.initial=0}
\pgfkeys{/tikz/tikzit category/.initial=0}

\pgfdeclarelayer{edgelayer}
\pgfdeclarelayer{nodelayer}
\pgfsetlayers{background,edgelayer,nodelayer,main}

\tikzstyle{none}=[inner sep=0mm]

\newcommand{\tikzfig}[1]{%
{\tikzstyle{every picture}=[tikzfig]
\IfFileExists{#1.tikz}
  {\input{#1.tikz}}
  {%
    \IfFileExists{./figures/#1.tikz}
      {\input{./figures/#1.tikz}}
      {\tikz[baseline=-0.5em]{\node[draw=red,font=\color{red},fill=red!10!white] {\textit{#1}};}}%
  }}%
}

\tikzstyle{every loop}=[]

\newtheorem{definition}{Definition}
\newtheorem{theorem}{Theorem}
\newtheorem{lemma}{Lemma}
\newtheorem{remark}{Remark}
\newtheorem{example}{Example}

\newcommand{\notion}[1]{\emph{\textbf{#1}}}

\newcommand{\oursgetcomm}[0]{\ensuremath{\mathsf{PRISM}}}
\newcommand{\hit}[0]{\ensuremath{\theta_{hit}}}
\newcommand{\sym}[0]{\ensuremath{\theta_{sym}}}
\newcommand{\js}[0]{\ensuremath{\theta_{JS}}}
\newcommand{\ntop}[0]{\ensuremath{n_{top}}}

\newcommand{\ourdim}[0]{\ensuremath{\delta}}
\newcommand{\ourspaths}[0]{\ensuremath{\mathsf{PathSymmetricClustering}}}
\newcommand{\hcluster}[0]{\ensuremath{\mathsf{HClustering}}}

\newcommand{\lbl}[1]{\ensuremath{\mathsf{label}(#1)}}

\newcommand{\E}{\mathbb{E}}
\newcommand{\N}{\mathcal{N}}
\newcommand{\Var}{\mathrm{Var}}

\newcommand{\Cov}{\mathrm{Cov}}

\makeatletter
\DeclareRobustCommand{\qed}{%
  \ifmmode % if math mode, assume display: omit penalty etc.
  \else \leavevmode\unskip\penalty9999 \hbox{}\nobreak\hfill
  \fi
  \quad\hbox{\qedsymbol}}
\newcommand{\openbox}{\leavevmode
  \hbox to.77778em{%
  \hfil\vrule
  \vbox to.675em{\hrule width.6em\vfil\hrule}%
  \vrule\hfil}}
\newcommand{\qedsymbol}{\openbox}
\newenvironment{proof}[1][\proofname]{\par
  \normalfont
  \topsep6\p@\@plus6\p@ \trivlist
  \item[\hskip\labelsep\itshape
    #1.]\ignorespaces
}{%
  \qed\endtrivlist
}
\newcommand{\proofname}{Proof}
\makeatother

\tikzstyle{swr}=[fill=white, draw=black, shape=rectangle]
\tikzstyle{swc}=[fill=white, draw=black, shape=circle, minimum size=0.5 cm, inner sep=0.02cm]
\tikzstyle{mwc}=[fill=white, draw=black, shape=circle, minimum size=0.75cm, inner sep=0.02cm]
\tikzstyle{oval}=[fill=white, draw=black, shape=circle, minimum height=0.5cm, minimum width=1 cm, ellipse, inner sep=0.02cm]
\tikzstyle{mbc}=[fill={rgb,255: red,191; green,191; blue,191}, draw=black, shape=circle, minimum size=0.75cm, inner sep=0.02cm]
\tikzstyle{basic}=[fill=white, draw=black, shape=circle]
\tikzstyle{square}=[fill=white, draw=black, shape=rectangle]
\tikzstyle{big dashed}=[fill=white, draw=black, shape=circle, minimum width=1cm, dashed]
\tikzstyle{vertical ellipse dashed}=[fill=none, draw=blue, minimum width=0.75cm, minimum height=3cm, ellipse, dashed, tikzit shape=rectangle, tikzit draw=blue, tikzit fill=white]
\tikzstyle{small vertical ellipse dashed}=[fill=none, draw=blue, shape=circle, tikzit fill=white, tikzit draw=blue, dashed, minimum width=0.75cm, minimum height=1.5cm, tikzit shape=rectangle, ellipse]
\tikzstyle{tiny vertical ellipse dashed}=[fill=none, draw=blue, shape=circle, tikzit fill=white, ellipse, dashed, minimum width=0.75cm, minimum height=1cm, tikzit shape=rectangle]
\tikzstyle{red}=[fill=red, draw=black, shape=circle]
\tikzstyle{green}=[fill={rgb,255: red,0; green,128; blue,128}, draw=black, shape=circle]
\tikzstyle{blue}=[fill=blue, draw=black, shape=circle]
\tikzstyle{huge dashed}=[fill=white, draw=black, shape=circle, dashed, minimum width=2cm]
\tikzstyle{medium}=[fill=white, draw=black, shape=circle, minimum width=1cm]
\tikzstyle{pale green}=[fill={rgb,255: red,173; green,231; blue,0}, draw=black, shape=circle, minimum width=1cm]
\tikzstyle{horizontal ellipse dashed}=[fill=white, draw=black, tikzit draw=magenta, tikzit shape=rectangle, minimum width=3cm, minimum height=0.75cm, ellipse, dashed]
\tikzstyle{minsize}=[fill=white, draw=black, shape=circle, minimum width=0.75cm]
\tikzstyle{horizontal ellipse green}=[fill={rgb,255: red,191; green,255; blue,0}, draw=black, tikzit draw={rgb,255: red,191; green,255; blue,0}, tikzit shape=rectangle, minimum width=3cm, minimum height=0.75cm, ellipse, dashed]
\tikzstyle{horizontal ellipse blue}=[fill={rgb,255: red,107; green,203; blue,255}, draw=black, tikzit draw=blue, tikzit shape=rectangle, minimum width=3cm, minimum height=0.75cm, ellipse, dashed]
\tikzstyle{smallblack}=[fill=black, draw=black, shape=circle, inner sep=0 pt, minimum size=3 pt]
\tikzstyle{smallSquare}=[fill=white, draw=black, shape=rectangle, inner sep=0 pt, minimum size=6 pt]
\tikzstyle{smallCircle}=[fill=white, draw=black, shape=circle, inner sep=0 pt, minimum size=6 pt]
\tikzstyle{big vertical ellipse dashed}=[fill=none, draw=blue, shape=circle, tikzit shape=rectangle, ellipse, dashed, minimum width=0.95cm, minimum height=3.7cm]
\tikzstyle{smallred}=[fill=red, draw=red, shape=circle, inner sep=0 pt, minimum size=3 pt]
\tikzstyle{redfilled}=[fill={red!20}, draw=red, shape=circle, opacity=0.5]
\tikzstyle{bluefilled}=[fill={blue!20}, draw=blue, shape=circle, opacity=0.5]
\tikzstyle{greenfilled}=[fill={rgb,255: red,149; green,255; blue,179}, draw={rgb,255: red,0; green,107; blue,61}, shape=circle, opacity=0.5]
\tikzstyle{orangefilled}=[fill={orange!20}, draw=orange, shape=circle, opacity=0.5]
\tikzstyle{new style 0}=[fill={rgb,255: red,191; green,191; blue,191}, draw=black, shape=circle]
\tikzstyle{small pink}=[fill=white, draw={rgb,255: red,18; green,162; blue,206}, shape=circle, inner sep=0 pt, minimum size=6 pt]
\tikzstyle{smallred}=[fill=white, draw=red, shape=circle, inner sep=0 pt, minimum size=6 pt]

\tikzstyle{sb}=[-]
\tikzstyle{new edge style 0}=[-, draw=red]
\tikzstyle{directed}=[->, -latex, draw=black]
\tikzstyle{undirected}=[-, line width=1pt, draw={rgb,255: red,128; green,128; blue,128}]
\tikzstyle{directed red}=[draw=red, ->, -latex]
\tikzstyle{directed green}=[draw={rgb,255: red,0; green,128; blue,128}, ->, line width=1pt]
\tikzstyle{directed blue}=[draw=blue, ->, line width=1pt]
\tikzstyle{directed purple}=[draw={rgb,255: red,128; green,0; blue,128}, ->, line width=1pt]
\tikzstyle{undirected red}=[-, draw=red]
\tikzstyle{undirected green}=[-, draw={rgb,255: red,0; green,107; blue,61}, line width=1pt]
\tikzstyle{undirected blue}=[-, draw=blue, line width=1pt]
\tikzstyle{undirected purple}=[-, draw={rgb,255: red,128; green,0; blue,128}, line width=1pt]
\tikzstyle{undirected dashed}=[-, line width=1pt, dashed, draw=black]
\tikzstyle{orange dashed}=[-, draw={rgb,255: red,255; green,128; blue,0}, dashed, line width=1.5pt]
\tikzstyle{directed dash}=[->, dashed]
\tikzstyle{blue dashed}=[-, draw=blue, dashed, line width=1pt]
\tikzstyle{green dashed}=[-, draw={rgb,255: red,0; green,162; blue,0}, dashed, line width=1pt]
\tikzstyle{blue filled}=[-, fill={blue!20}, draw=blue, line width=1pt, opacity=0.5, tikzit fill=white]
\tikzstyle{red filled}=[-, fill={red!20}, line width=1pt, draw=red, opacity=0.5, tikzit fill=white]
\tikzstyle{green filled}=[-, line width=1pt, draw={rgb,255: red,0; green,107; blue,61}, opacity=0.5, tikzit fill=white, fill={rgb,255: red,149; green,255; blue,179}]
\tikzstyle{orange filled}=[-, fill={orange!20}, draw=orange, line width=1pt, opacity=0.5, tikzit fill=white]
\tikzstyle{undirected dashed}=[-, draw={rgb,255: red,128; green,128; blue,128}, dashed, line width=1pt]
\tikzstyle{directed}=[->, -latex, fill=none, draw={rgb,255: red,128; green,128; blue,128}]
\tikzstyle{reddashed}=[-, draw=red, dashed]
\tikzstyle{blackdashed}=[-, dashed, draw=black]
\tikzstyle{black}=[-, ->, -latex, draw=black]
\tikzstyle{pinksolid}=[-, draw={rgb,255: red,18; green,162; blue,206}]
\tikzstyle{pinkdirected}=[-, ->, -latex, draw={rgb,255: red,18; green,162; blue,206}]
\tikzstyle{pink dashed}=[-, dashed, draw={rgb,255: red,18; green,162; blue,206}]
\usepackage{adjustbox}
\usepackage{wrapfig}
\usepackage{enumitem}

\usepackage{newfloat}
\DeclareCaptionStyle{ruled}{labelfont=normalfont,labelsep=colon,strut=off} % DO NOT CHANGE THIS
\title{Principled and Efficient Motif Finding for\\ Structure Learning of Lifted Graphical Models}
\author{
    Jonathan Feldstein\equalcontrib \textsuperscript{\rm 1,2},
    Dominic Phillips\equalcontrib \textsuperscript{\rm 1},
    Efthymia Tsamoura\textsuperscript{\rm 3}
}
\affiliations{
    \textsuperscript{\rm 1} University of Edinburgh, Edinburgh, United Kingdom
    \\
    \textsuperscript{\rm 2} BENNU.AI, Edinburgh, United Kingdom\\
    \textsuperscript{\rm 3} Samsung AI, Cambridge, United Kingdom\\
    jonathan.feldstein@bennu.ai, dominic.phillips@ed.ac.uk,
    efi.tsamoura@samsung.com
}

\begin{document}

\maketitle

\begin{abstract}
\textit{Structure learning} is a core problem in AI central to the fields of \textit{neuro-symbolic AI} and \textit{statistical relational learning}. It consists in automatically learning a logical theory from data. The basis for structure learning is mining repeating patterns in the data, known as \textit{structural motifs}. Finding these patterns reduces the exponential search space and therefore guides the learning of formulas. Despite the importance of motif learning, it is still not well understood. We present the first principled approach for mining structural motifs in \textit{lifted graphical models}, 
languages that blend first-order logic with probabilistic models,  which uses a stochastic process to measure the similarity of entities in the data. 

Our first contribution is an algorithm, which depends on two intuitive hyperparameters: one controlling the uncertainty in the entity similarity measure, and one controlling the softness of the resulting rules. Our second contribution is a preprocessing step where we perform hierarchical clustering on the data to reduce the search space to the most relevant data. Our third contribution is to introduce an $\mathcal{O}(n\ln{n})$ (in the size of the entities in the data) algorithm for clustering structurally-related data. We evaluate our approach using standard benchmarks and show that we outperform state-of-the-art structure learning approaches by up to 6\% in terms of accuracy and up to 80\% in terms of runtime.
\end{abstract}

%%%%%%%%%%%%% INTRODUCTION %%%%%%%%%%%%%%%%%%%%%%%
\section{Introduction}

\textbf{Motivation}
In artificial intelligence, combining statistical and logical representations is a long-standing and challenging aim. The motivation behind combining the two is that logical models can represent heterogenous data and capture causality, while statistical models handle uncertainty \cite{introduction-to-statistical-relational-learning,russell2015unifying}. General approaches to represent structural information are \textit{lifted graphical models} (LGMs), such as \emph{Markov logic networks} (MLNs) \cite{richardson_markov_2006} and \emph{probabilistic soft logic} (PSL) \cite{psl}. These are languages that define Markov random fields in a declarative fashion and are represented as theories of weighted formulas in first-order logic. The versatility of LGMs is reflected in their variety of applications, including bioinformatics \cite{lippi_prediction_2009}, natural language understanding\cite{wu_automatically_2008}, entity linking \cite{singla_entity_2006} and others \cite{chen_knowledge_2014, ha_goal_2011, riedel_collective_2008, crane_investigating_2012}. Recently, they have also been adopted in neurosymbolic frameworks \cite{dpl, ts}.

Unsurprisingly, the quality of a logical theory, that is, the extent to which it models the task it is supposed to solve, has a strong impact on the performance of the downstream applications. Manually optimising formulae to boost performance is a costly, time-consuming and error-prone process that restricts the scope of application. This can raise fundamental criticism against frameworks that require such theories as part of their input \cite{deepproblog,neurasp}. An alternative is the automated learning of LGMs from data, a problem known as \emph{structure learning}. The ultimate goal is to design a general framework that can efficiently learn high-quality models on large datasets in a principled fashion. Several pioneering structure learning algorithms have been developed for MLNs \cite{kok_learning_2005, mihalkova_bottom-up_2007, kok_learning_nodate}. 

\textbf{Problem}
Generally, structure learning consists in searching for formulae in an exponential search space. The naive approach would consist in trying every possible combination of predicates and logical connectives, which is computationally expensive \cite{kok_learning_2005}. Therefore, to reduce computational complexity, every sophisticated structure learner proceeds by searching for formulae within templates. These templates can be user-defined or learnt automatically \cite{mihalkova_bottom-up_2007, kok_alchemy_2010}. 
Every sophisticated learner can thus be summarized in three main steps: S1 - Apply heuristics to abstract-out common, recurrent patterns within the data to be used as templates. S2 - Iteratively generate formulae based on the previously found patterns and evaluate candidate formulae based on how well they generalize to the training data. S3 - Learn the collective weights of the optimal formulae. Remark that finding good templates is the basis for successful structural learning, as it not only reduces the search space but also forms the starting point of the structure learning algorithm and constrains the shape of logical formulae generated in later stages.

For example, the state-of-the-art structure learning algorithm, \textit{Learning using Structural Motifs} (LSM), reduces the search space for formulae by focusing within recurring patterns of commonly connected entities in the relational database \cite{kok_learning_nodate}. The task of mining these patterns involves repeatedly partitioning the entities of a database into symmetrically-equivalent sets relative to a reference entity. These sets are called \textit{(structural) motifs}. Since the entities in a structural motif are symmetric, formula learning only needs to be performed on one entity instead of each separately. Therefore, structural motifs guide the groundings of potential logical formulas of the LGM (S1). 

The key difference between structure learners that do not require user input is how the templates are found. Still, the state-of-the-art suffers from several shortcomings that have a negative impact on the scalability and effectiveness of the full pipeline (S1-S3). Firstly, the symmetry-partitioning algorithm has six unintuitive hyperparameters that need to be calibrated to each dataset. The difficulty of finding these parameters can lead to inaccurate partitioning. Secondly, the main clustering algorithm, a core step to obtain symmetric partitions, has $\mathcal{O}(n^3)$ complexity in the number of entities to partition. This can result in significant slowdowns on databases that are densely connected.  

\textbf{Contributions}
In this work, we design a more principled and scalable algorithm for extracting motifs through symmetry-partitioning (stage S1 of structure learning). In our algorithm, we make three key contributions that overcome the limitations of prior art. In Section \ref{section:ours:parameters}, we address the first limitation and propose a \textit{principled} algorithm using the theoretic properties of hypergraphs to design an approach that uses just two, intuitive hyperparameters: one that controls the uncertainty of the similarity measure of entities in the data and one that controls the softness of the resulting formulae. In Section \ref{section:ours:path-symmetry}, we tackle the issue of \textit{efficiency}. Firstly, we design an alternative $\mathcal{O}(n \ln{n})$ symmetry-partitioning algorithm. Secondly, we propose a pre-processing step where we hierarchically cluster the relational database to reduce the required computation and further improve the guiding of the formulae finding. Beyond the above contributions, we present PRISM (PRincipled Identification of Structural Motifs) a parallelized, flexible, and optimized C++ implementation of the entire algorithm\footnote{https://github.com/jonathanfeldstein/PRISM}. In Section \ref{sec:experiments}, we assess the performance of the developed techniques against LSM and BOOSTR on datasets used as standard benchmarks in the literature. 
  
%%%%%%%%%%%%%%% PRELIMINARIES %%%%%%%%%%%%%%%%%%%%%%%%
\section{Preliminaries}\label{section:preliminaries}

A \notion{hypergraph} $\mathcal{H}=(V,E)$ is a pair of sets of nodes $V = \{v_i\}_{i=0}^{\vert V \vert}$ and hyperedges $E = \{e_i\}_{i=0}^{\vert E \vert}$. A hyperedge $e_i \in E$ is a non-empty subset of the nodes in $\mathcal{H}$. A hypergraph $\mathcal{H}$ is labelled, if each hyperedge in $\mathcal{H}$ is labelled with a categorical value. We use $\lbl{e_i}$ to denote the 
label of the hyperedge $e_i$.
A \notion{path} $\pi$ of length $L$ in $\mathcal{H}$ is an alternating sequence of nodes $v_i$ and hyperedges $e_i$, such that
$v_i, v_{i+1} \in e_{i}$, of the form $(v_0,e_{0},v_1,\dots,v_{L-1},e_{L-1},v_L)$ for ${0 \leq i \leq L-1}$. 
The \notion{diameter} of $\mathcal{H}$ is the maximum length of the shortest path between any two nodes $v_i, v_j$ in $\mathcal{H}$. 
The \notion{signature} of a path $\pi$ is the sequence of the labels of the edges occurring in 
$\pi$, i.e., $(\lbl{e_{0}},\dots,\lbl{e_{L-1}})$.

A \notion{relational database} $\mathcal{D}$ can be represented by a hypergraph $\mathcal{H}=(V,E)$ by defining $V$ to be the union of the constants in $\mathcal{D}$, and defining $E$ such that every $k$-ary ground atom ${R(c_1,\dots,c_k)}$ in $\mathcal{D}$ becomes a hyperedge $e \in E$, with label $R$, whose elements are the nodes corresponding to the constants ${c_1,\dots,c_n}$. 

A \notion{random walk} on $\mathcal{H}$ is a stochastic process that generates paths by traversing edges in $\mathcal{H}$.
The \notion{length} of a random walk is defined as the number of edges traversed in the path.
Let $v_i$ and $v_j$ be two nodes in $\mathcal{H}$.
The \notion{hitting time} $h_{i,j}$ from $v_i$ to $v_j$ is the average number of steps required to reach $v_j$ for the first time with random walks starting from $v_i$. 

The \notion{L-truncated hitting time} $h^{L}_{ij}$ (THT) is the hitting time where the length of the random walk is limited to at most $L$ steps. It is defined recursively as $h^{L}_{ij} = 1 + \sum_k p_{ik}h^{L-1}_{kj}$, where $p_{ij}$ is the transition matrix of the random walk, with $h^{L}_{ij}=0$ if $i=j$, and $h^{L}_{ij} = L$ if $j$ is not reached in $L$ steps. The more short paths that exist between $v_i$ and $v_j$, the shorter the THT. The THT is therefore a measure of the connectedness of nodes.

We denote by $\mathcal{S}^L_{i,j}$ the set of path signatures of lengths up to $L$ that start at $v_i$ and end at $v_j$.
The $L$-\notion{path signature distribution} $P^L_{i,j}$ is 
then the probability distribution over the elements of $\mathcal{S}^L_{i,j}$ under a given random walk process. The \notion{marginal $L$-path signature distribution} $P^{L}_{i,j}\vert_l$ is the marginal probability distribution when only paths of length \textit{exactly} $l \in \{1, 2, \dots, L\}$ are considered.
The quantities $P^L_{i,j}(\sigma)$ and $P^L_{i,j}\vert_l(\sigma)$ respectively denote the probability and marginal probability of path signature $\sigma$. With this, we now introduce the important notion of \notion{path-symmetry}.

\begin{definition} [Path-Symmetry] \label{definition:path-symmetric}
	Nodes $v_j$ and $v_k$ are order-$L$ path symmetric with respect to $v_i$ if
	$P^L_{i,j} = P^L_{i,k}$ and are exact order-$l$ path symmetric w.r.t. $v_i$ if $P^L_{i,j}\vert_l = P^L_{i,k}\vert_l$. A set of nodes is (exact) path-symmetric w.r.t $v_i$ if each node in the set is (exact) path-symmetric w.r.t. $v_i$. 
\end{definition}

Within the context of structure learning, path-symmetric sets of nodes correspond to what we denote as \notion{abstract concepts} and correspond to collections of entities that have similar neighbourhoods in the hypergraph.

\begin{remark}
\label{distpathsymremark}
A necessary condition for nodes $v_j$ and $v_k$ to be order-$L$ path-symmetric w.r.t. $v_i$ is that they are \notion{order-$L$ distance symmetric} w.r.t. $v_i$, i.e. $h^L_{i,j} = h^L_{i,k}$.
\end{remark}

It is computationally infeasible to compute $h^{L}_{i,j}$ and $P^L_{i,j}$ exactly for large hypergraphs. However, they can both be well-approximated by sampling by running $N$ random walks of length $L$ from node $v_i$ and recording the number of times $v_j$ is hit \cite{sarkar_fast_2008}. We denote by $\hat{h}^{L,N}_{i,j}$,
and by $\hat{P}^{L,N}_{i,j}$, the obtained estimates and refer to them as $(L,N)$ estimates. 
Finally, we denote by $\hat{C}^{L,N}_{i,j}(\sigma)$ the function from a signature $\sigma$ in $\mathcal{S}^L_{i,j}$ to the number of occurrences of $\sigma$ in the paths from $v_i$ to $v_j$ that are encountered while running $N$ random walks of length $L$. We refer to $\hat{C}^{L,N}_{i,j}(\sigma)$ as the $L$-\notion{path signature counts}. 

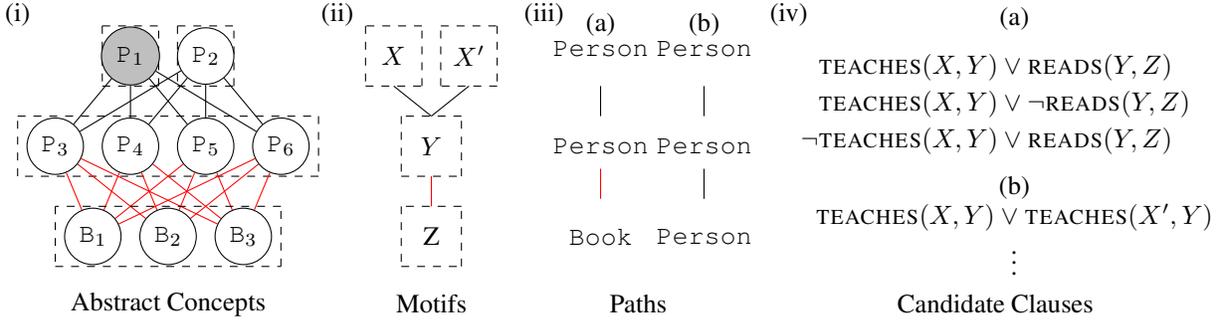
\begin{figure*}[t]
\centering
\begin{tikzpicture}
	\begin{pgfonlayer}{nodelayer}
		\node [style=mwc] (0) at (-3.5, 0.125) {$\texttt{P}_4$};
		\node [style=mwc] (1) at (-2.5, 0.125) {$\texttt{P}_5$};
		\node [style=mwc] (2) at (-1.5, 0.125) {$\texttt{P}_6$};
		\node [style=mwc] (3) at (-4.5, 0.125) {$\texttt{P}_3$};
		\node [style=new style 0] (4) at (-3.5, 1.325) {$\texttt{P}_1$};
		\node [style=mwc] (5) at (-2.5, 1.325) {$\texttt{P}_2$};
		\node [style=none] (6) at (-5, 0.525) {};
		\node [style=none] (7) at (-5, -0.275) {};
		\node [style=none] (8) at (-1, -0.275) {};
		\node [style=none] (9) at (-1, 0.525) {};
		\node [style=none] (10) at (-3.875, 1.725) {};
		\node [style=none] (11) at (-3.125, 1.725) {};
		\node [style=none] (12) at (-3.875, 0.925) {};
		\node [style=none] (13) at (-3.125, 0.925) {};
		\node [style=none] (14) at (-2.125, 0.925) {};
		\node [style=none] (15) at (-2.125, 1.725) {};
		\node [style=none] (16) at (-0.375, 1.725) {};
		\node [style=none] (17) at (-0.375, 0.925) {};
		\node [style=none] (18) at (0.375, 0.925) {};
		\node [style=none] (19) at (0.375, 1.725) {};
		\node [style=none] (20) at (0.1, 0.525) {};
		\node [style=none] (21) at (0.1, -0.275) {};
		\node [style=none] (22) at (0.9, -0.275) {};
		\node [style=none] (23) at (0.9, 0.525) {};
		\node [style=none] (24) at (0.5, 0.125) {$Y$};
		\node [style=none] (25) at (0, 1.325) {$X$};
		\node [style=none] (26) at (0, 0.925) {};
		\node [style=none] (27) at (0.5, 0.525) {};
		\node [style=none] (30) at (-3, -1.975) {Abstract Concepts};
		\node [style=none] (31) at (0.5, -1.975) {Motifs};
		\node [style=mwc] (32) at (-4, -1.075) {$\texttt{B}_1$};
		\node [style=mwc] (33) at (-3, -1.075) {$\texttt{B}_2$};
		\node [style=mwc] (34) at (-2, -1.075) {$\texttt{B}_3$};
		\node [style=none] (35) at (-4.5, -0.675) {};
		\node [style=none] (36) at (-4.5, -1.475) {};
		\node [style=none] (37) at (-1.5, -1.475) {};
		\node [style=none] (38) at (-1.5, -0.675) {};
		\node [style=none] (39) at (0.1, -0.675) {};
		\node [style=none] (40) at (0.9, -0.675) {};
		\node [style=none] (41) at (0.1, -1.475) {};
		\node [style=none] (42) at (0.9, -1.475) {};
		\node [style=none] (43) at (0.5, -0.275) {};
		\node [style=none] (44) at (0.5, -1.075) {Z};
		\node [style=none] (48) at (2.75, 1.425) {\texttt{Person}};
		\node [style=none] (49) at (2.75, 0.125) {\texttt{Person}};
		\node [style=none] (50) at (2.75, -1.075) {\texttt{Book}};
		\node [style=none] (52) at (0.5, -0.675) {};
		\node [style=none] (53) at (2.75, 0.925) {};
		\node [style=none] (54) at (2.75, 0.525) {};
		\node [style=none] (55) at (2.75, -0.175) {};
		\node [style=none] (56) at (2.75, -0.575) {};
		\node [style=none] (57) at (3.25, -1.975) {Paths};
		\node [style=none] (60) at (8, 0.675) {{$\begin{aligned}\textsc{teaches}(X,Y) &\vee \textsc{reads}(Y,Z)\\ \textsc{teaches}(X,Y) &\vee \neg\textsc{reads}(Y,Z)\\ \neg\textsc{teaches}(X,Y) &\vee \textsc{reads}(Y,Z)\end{aligned}$}};
		\node [style=none] (61) at (8, -1.975) {Candidate Clauses};
		\node [style=none] (64) at (-5, 1.875) {(i)};
		\node [style=none] (65) at (-0.75, 1.875) {(ii)};
		\node [style=none] (66) at (2, 1.875) {(iii)};
		\node [style=none] (67) at (5.25, 1.875) {(iv)};
		\node [style=none] (72) at (-2.875, 1.725) {};
		\node [style=none] (73) at (-2.875, 0.925) {};
		\node [style=none] (80) at (0.625, 0.925) {};
		\node [style=none] (81) at (1, 0.925) {};
		\node [style=none] (82) at (0.625, 1.725) {};
		\node [style=none] (83) at (1.375, 1.725) {};
		\node [style=none] (84) at (1.375, 0.925) {};
		\node [style=none] (85) at (1, 1.325) {$X'$};
		\node [style=none] (86) at (4.125, 1.425) {\texttt{Person}};
		\node [style=none] (87) at (4.125, 0.125) {\texttt{Person}};
		\node [style=none] (88) at (4.125, -1.075) {\texttt{Person}};
		\node [style=none] (89) at (4.125, 0.925) {};
		\node [style=none] (90) at (4.125, 0.525) {};
		\node [style=none] (91) at (4.125, -0.175) {};
		\node [style=none] (92) at (4.125, -0.575) {};
		\node [style=none] (94) at (2.75, 1.75) {(a)};
		\node [style=none] (95) at (4.125, 1.75) {(b)};
		\node [style=none] (96) at (8.25, -0.825) {{$\textsc{teaches}(X,Y) \vee \textsc{teaches}(X',Y)$}};
		\node [style=none] (98) at (8.25, -1.325) {$\vdots$};
		\node [style=none] (99) at (8.25, 1.825) {(a)};
		\node [style=none] (100) at (8.25, -0.425) {(b)};
	\end{pgfonlayer}
	\begin{pgfonlayer}{edgelayer}
		\draw (4) to (3);
		\draw (4) to (0);
		\draw (4) to (1);
		\draw (4) to (2);
		\draw (5) to (2);
		\draw (5) to (1);
		\draw (5) to (0);
		\draw (5) to (3);
		\draw [style=blackdashed] (6.center) to (9.center);
		\draw [style=blackdashed] (9.center) to (8.center);
		\draw [style=blackdashed] (8.center) to (7.center);
		\draw [style=blackdashed] (7.center) to (6.center);
		\draw [style=blackdashed] (10.center) to (11.center);
		\draw [style=blackdashed] (10.center) to (12.center);
		\draw [style=blackdashed] (12.center) to (13.center);
		\draw [style=blackdashed] (13.center) to (11.center);
		\draw [style=blackdashed] (16.center) to (17.center);
		\draw [style=blackdashed] (17.center) to (18.center);
		\draw [style=blackdashed] (18.center) to (19.center);
		\draw [style=blackdashed] (19.center) to (16.center);
		\draw [style=blackdashed] (20.center) to (21.center);
		\draw [style=blackdashed] (21.center) to (22.center);
		\draw [style=blackdashed] (22.center) to (23.center);
		\draw [style=blackdashed] (23.center) to (20.center);
		\draw (26.center) to (27.center);
		\draw [style=undirected red] (3) to (32);
		\draw [style=undirected red] (0) to (33);
		\draw [style=undirected red] (1) to (33);
		\draw [style=undirected red] (0) to (32);
		\draw [style=undirected red] (1) to (34);
		\draw [style=undirected red] (2) to (34);
		\draw [style=undirected red] (2) to (33);
		\draw [style=undirected red] (3) to (33);
		\draw [style=undirected red] (1) to (32);
		\draw [style=undirected red] (0) to (34);
		\draw [style=undirected red] (2) to (32);
		\draw [style=undirected red] (3) to (34);
		\draw [style=blackdashed] (35.center) to (36.center);
		\draw [style=blackdashed] (36.center) to (37.center);
		\draw [style=blackdashed] (37.center) to (38.center);
		\draw [style=blackdashed] (38.center) to (35.center);
		\draw [style=blackdashed] (39.center) to (41.center);
		\draw [style=blackdashed] (41.center) to (42.center);
		\draw [style=blackdashed] (42.center) to (40.center);
		\draw [style=blackdashed] (40.center) to (39.center);
		\draw [style=undirected red] (43.center) to (52.center);
		\draw [style=undirected red, in=90, out=-90] (55.center) to (56.center);
		\draw [in=90, out=-90] (53.center) to (54.center);
		\draw [style=blackdashed, in=180, out=0] (73.center) to (14.center);
		\draw [style=blackdashed] (14.center) to (15.center);
		\draw [style=blackdashed] (72.center) to (15.center);
		\draw [style=blackdashed, in=90, out=-90] (72.center) to (73.center);
		\draw [style=blackdashed] (83.center) to (84.center);
		\draw [style=blackdashed] (84.center) to (80.center);
		\draw [style=blackdashed] (80.center) to (82.center);
		\draw [style=blackdashed] (82.center) to (83.center);
		\draw (81.center) to (27.center);
		\draw [in=90, out=-90] (91.center) to (92.center);
		\draw (89.center) to (90.center);
	\end{pgfonlayer}
\end{tikzpicture}
\caption{Example Structure-Learning Pipeline: The above shows a dataset about a university class. Nodes $\texttt{P}_i$ are entities of type \texttt{person}, while $\texttt{B}_i$ are entities of type $\texttt{book}$. Black edges represent $\textsc{teaches}(\texttt{person},\texttt{person})$, and red edges represent $\textsc{reads}(\texttt{person},\texttt{book})$: (i) The resulting abstract concepts when random walks are run from source node $\texttt{P}_1$. Dashed boxes represent concepts, which intuitively are teachers, $\{\texttt{P}_1\}$, colleagues $\{\texttt{P}_2\}$, students $\{\texttt{P}_3, \texttt{P}_4, \texttt{P}_5, \texttt{P}_6\}$ and books $\{\texttt{B}_1, \texttt{B}_2, \texttt{B}_3\}$ (ii) the resulting structural motif, (iii) paths found in the motif (iv) mined candidate clauses.} \label{fig:hc_concept}
\end{figure*}
 
We denote by $\hat{C}^{L,N}_{i,j}\vert_l(\sigma)$ the marginal count when only contributions from paths of length exactly $l \in \{1, 2, \dots, L\}$ are considered.  
For readability purposes, we will drop the term signature and refer simply to $L$-\notion{path distributions} and to $L$-\notion{path counts}. 
By path, we will refer to a path signature, unless stated otherwise. 

\subsection{Example of a Structure Learner: LSM} \label{section:preliminaries:lsm}

To illustrate structure learning, we present an overview of the LSM algorithm \cite{kok_learning_nodate}. The algorithm proceeds in three main steps (denoted S1, S2 and S3 below). The resulting pipeline is summarised in Fig \ref{fig:hc_concept}.

\textbf{S1: Finding Structural Motifs} Nodes with similar environments in the database hypergraph are first clustered together into \textit{abstract concepts}. Clustering is achieved by running many random walks from each node in the hypergraph. Nodes are then partitioned into sets of path-symmetric nodes based on the similarity of their $L$-path counts. Each path-symmetric sets then corresponds to an abstract concept.
\begin{example}[Abstract Concepts]
In Fig \ref{fig:hc_concept}, we see that $\texttt{P}_1$ and $\texttt{P}_2$ are both teaching $\texttt{P}_3$, $\texttt{P}_4$, $\texttt{P}_5$ and $\texttt{P}_6$. Furthermore, $\texttt{P}_3$, $\texttt{P}_4$, $\texttt{P}_5$ and $\texttt{P}_6$ are all reading \texttt{B$_1$}, \texttt{B$_2$} and \texttt{B$_3$}. Even though we have not explicitly defined the notion of \textit{student}, \textit{teacher}, and \textit{book}  we have that $\texttt{P}_3$, $\texttt{P}_4,$ $\texttt{P}_5$ and $\texttt{P}_6$ are all path-symmetric w.r.t to $\texttt{P}_1$ and w.r.t $\texttt{P}_2$, as are \texttt{B$_1$}, \texttt{B$_2$} and \texttt{B$_3$}. The abstract concepts that we obtain are thus $\{\texttt{P}_3, \texttt{P}_4, \texttt{P}_5, \texttt{P}_6\}$, $\{\texttt{P}_1, \texttt{P}_2\}$, and $\{\texttt{B$_1$}, \texttt{B$_2$},\texttt{B$_3$}\}$, which intuitively represent the idea of students, teachers and books, respectively.
\end{example}
Once abstract concepts are found, they are then joined by the edges that connect them to form \textit{structural motifs}, see Fig \ref{fig:hc_concept} (ii). It is the identification of these structural motifs that effectively speeds up the subsequent rule-finding by reducing the search for candidate clauses (c.f. S2). In LSM, computing motifs requires setting six independent hyper-parameters: $N$ the number of random walks ran, $L$ the length of each random walk,  $\hit$ a threshold to select only `nearby' nodes to the source node of the random walk (those with $\hat{h}^{L,N}_{i,j} \leq \hit$), $\sym$ a threshold for merging nodes based on the similarity of their THTs (all nodes $v_j$ and $v_k$ with $|\hat{h}^{L,N}_{i,j} - \hat{h}^{L,N}_{i,k}| < \sym$ are merged), $\js$ a threshold for merging nodes by path similarity based on the Jensen-Shannon divergence of their path distributions, and $\ntop$ the number of paths to consider (in order of descending frequency) when computing the Jensen-Shannon divergence.

\textbf{S2a: Finding Paths in Motifs} Based on the found motifs, sequences (paths) of ground literals that often appear together in the data are generated, see Fig \ref{fig:hc_concept} (iii). The fact that the literals appear often together points to the fact that they are likely to be logically dependent on one another.

\textbf{S2b: Evaluating Candidate Clauses} The sequences of ground literals are used to generate candidate clauses. Each clause is evaluated using a likelihood function. The best clauses are then added to the structure-learnt MLN.

\textbf{S3: Learning the Weights of Candidate Clauses} Finally, the algorithm finds the weights of the chosen clauses by maximum-likelihood estimation. This yields a set of formula-weight pairs which define the final MLN.

%%%%%%%%%%%% CORE %%%%%%%%%%%%%%%%%%%%%%%%%%%%%%%%

\section{Principled Motif-Finding}\label{section:ours:parameters}

Hyperparameter tuning can be one of the most time-costly stages when applying algorithms to real problems. This holds particularly in the case of LSM, where we have six heuristic hyperparameters, as detailed in Section \ref{section:preliminaries:lsm}. In our work, we devise an alternative motif-finding algorithm (PRISM) that depends on only two \textit{intuitive} hyperparameters, thus greatly speeding up the workflow.

\subsection{Introducing PRISM}
In overview, the steps taken by PRISM are:

For each node $v_i$ in $\mathcal{H}$: (i) run an \textit{optimal} number of random walks originating from $v_i$ and compute, for each $v_j \neq v_i$, the THT estimate $\hat{h}^{L,N}_{i,j}$ and path distribution estimate $\hat{P}^{L,N}_{i,j}$; (ii) partition the nodes $V \in \mathcal{H}$ into sets $A_1,A_2,\dots,A_M$, that are \textit{statistically significant} order-$L$ distance-symmetric w.r.t. $v_i$, by merging nodes if the difference in their THTs is below a statistical threshold $\sym$. We describe how to set $\sym$ in Section \ref{section:ours:significance}; (iii) further partition the nodes within each $A_m$ into \textit{statistically significant} order-$L$ path-symmetric sets. An algorithm achieving this in $\mathcal{O}(n \ln{n})$ (vs $\mathcal{O}(n^3)$ in SOTA) is presented later. 
Notice that step (ii) serves to reduce the computational cost of step (iii) by applying heuristics that classify the nodes into sets that are most likely to be path-symmetric. 

The question remains how to define `optimal' and `statistically significant' in the above pipeline. To this end, we introduce two independent parameters, $\varepsilon$ to optimise 
the number of random walks, and $\alpha$ to control the statistical significance threshold of the similarity measure.

\subsection{$\varepsilon$-uncertainty: Controlled Path Sampling}\label{section:ours:uncertainty}

\textbf{Motivation} To find good motifs we need to identify abstract concepts. To do this, we compare the path distributions of nodes in the hypergraph representation of the database. However, in real-world applications, computing these distributions exactly is infeasible, so we resort to approximating them through sampling by running random walks. The uncertainty in these approximations will depend on the length $L$, and number $N$ of random walks. Here we formally define a measure of uncertainty and show how it can be used to set an optimal number of random walks.

\begin{definition}[$\varepsilon$-uncertainty]\label{definition:uncertainty}
The uncertainty of the $(L,N)$-estimate of ${h}_{i,j}$ is defined by ${|{h}^L _{i,j} - \hat{h}^{L,N}_{i,j}|}/{{h}^L_{i,j}}$. 
The uncertainty of the $(L,N)$-estimate of ${P}^L_{i,j}$ is defined as the maximum of ${|{P}^L_{i,j}(\sigma) - \hat{P}^{L,N}_{i,j}(\sigma)|}/{{P}^L_{i,j}(\sigma)}$ among all paths $\sigma$ in the domain of ${P}^L_{i,j}$.
\end{definition}
$\varepsilon$-uncertainty is of major importance to the overall theory-induction pipeline as it determines the confidence in the similarity measure between nodes and, ultimately, of the induced theories; the lower the $\varepsilon$, the higher the confidence in the relatedness of nodes. However, naturally, there is a trade-off, as we show below (Thm. \ref{theorem:optimality}), as lower uncertainty implies a polynomially higher computational cost. For a given $\varepsilon$ we thus seek to find the least number of random walks $N$ that guarantees this uncertainty level. We say that such an $N$ is \textit{$\varepsilon$-optimal}: 

\begin{definition}[$\varepsilon$-optimality]
$N$ is $\varepsilon$-optimal on $\mathcal{H}$ under $L$ if it is the smallest integer so that for any pair of nodes $v_i,v_j$ in $\mathcal{H}$, the expectation of the uncertainties of $(L,N)$-estimates of $h_{i,j}$ and $P_{i,j}$ are upper bounded by $\varepsilon$. 
\end{definition}

Minimising $N$ is crucial as running random walks is computationally intensive, especially in large hypergraphs.

\textbf{Usage} In Theorem~ \ref{theorem:optimality} below, we state how to set $N$ to guarantee $\varepsilon$-optimality (for all theorem proofs, see the Appendix). 

\begin{theorem} \label{theorem:optimality}
An upper bound on the $\varepsilon$-optimal number of random walks $N$ on $\mathcal{H}$ under $L$ is given by 
    \begin{align}
        \max\{{(L-1)^2}/{4\varepsilon^2} , {P^{*}\left(\gamma + \ln P^{*}\right)}/{\varepsilon^2}  \}
    \end{align}
    where ${P^{*} = 1 + {e\left(e^{L}-1\right)}/({e-1}) \gg 1}$, $e$ is the number of unique edge labels in $\mathcal{H}$, and $\gamma \approx 0.577$ is the Euler-Mascheroni constant. 
\end{theorem}
In PRISM, $N$ is automatically computed according to Theorem \ref{theorem:optimality} based on a user-specified $\varepsilon$. In the above, we assumed a fixed $L$. A good value for $L$ is the diameter of the hypergraph to guarantee that every node can be hit during random walks. In Section \ref{section:h-clustering} we revise this assumption and show how $L$ can be reduced based on the properties of $\mathcal{H}$.

\subsection{$\alpha$-significance: Controlled Softness of Formulae}\label{section:ours:significance}

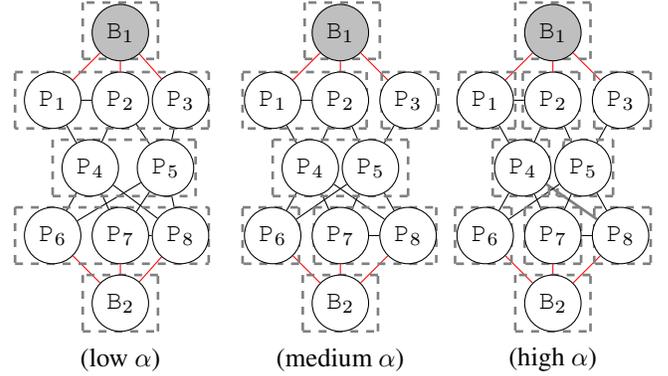
\begin{figure}[htb]
    \centering
    \begin{tikzpicture}
	\begin{pgfonlayer}{nodelayer}
		\node [style=mwc] (0) at (-3.325, 0) {$\texttt{P}_4$};
		\node [style=mwc] (1) at (-2.325, 0) {$\texttt{P}_5$};
		\node [style=mwc] (2) at (-3.825, 0.9) {$\texttt{P}_1$};
		\node [style=mwc] (3) at (-2.925, 0.9) {$\texttt{P}_2$};
		\node [style=mwc] (4) at (-2.125, 0.9) {$\texttt{P}_3$};
		\node [style=mbc] (5) at (-2.925, 1.8) {$\texttt{B}_1$};
		\node [style=mwc] (6) at (-3.825, -0.9) {$\texttt{P}_6$};
		\node [style=mwc] (7) at (-2.925, -0.9) {$\texttt{P}_7$};
		\node [style=mwc] (8) at (-2.125, -0.9) {$\texttt{P}_8$};
		\node [style=mwc] (9) at (-2.925, -1.8) {$\texttt{B}_2$};
		\node [style=none] (24) at (-3.825, 0.375) {};
		\node [style=none] (27) at (-3.825, -0.375) {};
		\node [style=none] (28) at (-1.925, 0.375) {};
		\node [style=none] (29) at (-1.925, -0.375) {};
		\node [style=none] (30) at (-4.325, 1.275) {};
		\node [style=none] (31) at (-1.75, 1.275) {};
		\node [style=none] (32) at (-1.75, 0.525) {};
		\node [style=none] (33) at (-4.325, 0.525) {};
		\node [style=none] (34) at (-4.325, -0.525) {};
		\node [style=none] (35) at (-4.325, -1.275) {};
		\node [style=none] (36) at (-1.75, -1.275) {};
		\node [style=none] (37) at (-1.75, -0.525) {};
		\node [style=none] (38) at (-3.425, -1.425) {};
		\node [style=none] (39) at (-2.425, -1.425) {};
		\node [style=none] (40) at (-2.425, -2.175) {};
		\node [style=none] (41) at (-3.425, -2.175) {};
		\node [style=none] (77) at (-1.625, 0) {};
		\node [style=none] (79) at (-2.925, -2.55) {(low $\alpha$)};
		\node [style=none] (80) at (0, -2.55) {(medium $\alpha$)};
		\node [style=mwc] (81) at (-0.4, 0) {$\texttt{P}_4$};
		\node [style=mwc] (82) at (0.4, 0) {$\texttt{P}_5$};
		\node [style=mwc] (83) at (-0.9, 0.9) {$\texttt{P}_1$};
		\node [style=mwc] (84) at (0, 0.9) {$\texttt{P}_2$};
		\node [style=mwc] (85) at (0.9, 0.9) {$\texttt{P}_3$};
		\node [style=mbc] (86) at (0, 1.8) {$\texttt{B}_1$};
		\node [style=mwc] (87) at (-0.9, -0.9) {$\texttt{P}_6$};
		\node [style=mwc] (88) at (0, -0.9) {$\texttt{P}_7$};
		\node [style=mwc] (89) at (0.9, -0.9) {$\texttt{P}_8$};
		\node [style=mwc] (90) at (0, -1.8) {$\texttt{B}_2$};
		\node [style=none] (91) at (-0.9, 0.375) {};
		\node [style=none] (92) at (-0.9, -0.375) {};
		\node [style=none] (93) at (0.9, 0.375) {};
		\node [style=none] (94) at (0.9, -0.375) {};
		\node [style=none] (95) at (-1.3, 1.275) {};
		\node [style=none] (96) at (0.35, 1.275) {};
		\node [style=none] (97) at (0.35, 0.525) {};
		\node [style=none] (98) at (-1.3, 0.525) {};
		\node [style=none] (99) at (-0.35, -0.525) {};
		\node [style=none] (100) at (-0.35, -1.275) {};
		\node [style=none] (101) at (1.275, -1.275) {};
		\node [style=none] (102) at (1.275, -0.525) {};
		\node [style=none] (103) at (-0.5, -1.425) {};
		\node [style=none] (104) at (0.5, -1.425) {};
		\node [style=none] (105) at (0.5, -2.175) {};
		\node [style=none] (106) at (-0.5, -2.175) {};
		\node [style=none] (107) at (1.5, 0) {};
		\node [style=none] (108) at (2.825, -2.55) {(high $\alpha$)};
		\node [style=mwc] (109) at (2.425, 0) {$\texttt{P}_4$};
		\node [style=mwc] (110) at (3.225, 0) {$\texttt{P}_5$};
		\node [style=mwc] (111) at (1.925, 0.9) {$\texttt{P}_1$};
		\node [style=mwc] (112) at (2.825, 0.9) {$\texttt{P}_2$};
		\node [style=mwc] (113) at (3.725, 0.9) {$\texttt{P}_3$};
		\node [style=mbc] (114) at (2.825, 1.8) {$\texttt{B}_1$};
		\node [style=mwc] (115) at (1.925, -0.9) {$\texttt{P}_6$};
		\node [style=mwc] (116) at (2.825, -0.9) {$\texttt{P}_7$};
		\node [style=mwc] (117) at (3.725, -0.9) {$\texttt{P}_8$};
		\node [style=mwc] (118) at (2.825, -1.8) {$\texttt{B}_2$};
		\node [style=none] (131) at (2.325, -1.425) {};
		\node [style=none] (132) at (3.325, -1.425) {};
		\node [style=none] (133) at (3.325, -2.175) {};
		\node [style=none] (134) at (2.325, -2.175) {};
		\node [style=none] (135) at (4.225, 0) {};
		\node [style=none] (136) at (0.55, 1.275) {};
		\node [style=none] (137) at (1.3, 1.275) {};
		\node [style=none] (138) at (1.3, 0.525) {};
		\node [style=none] (139) at (0.55, 0.525) {};
		\node [style=none] (140) at (-1.3, -0.525) {};
		\node [style=none] (141) at (-0.55, -0.525) {};
		\node [style=none] (142) at (-0.55, -1.275) {};
		\node [style=none] (143) at (-1.3, -1.275) {};
		\node [style=none] (144) at (1.575, 1.275) {};
		\node [style=none] (145) at (2.325, 1.275) {};
		\node [style=none] (146) at (2.325, 0.525) {};
		\node [style=none] (147) at (1.575, 0.525) {};
		\node [style=none] (148) at (2.425, 1.275) {};
		\node [style=none] (149) at (3.175, 1.275) {};
		\node [style=none] (150) at (3.175, 0.525) {};
		\node [style=none] (151) at (2.425, 0.525) {};
		\node [style=none] (152) at (3.375, 1.275) {};
		\node [style=none] (153) at (4.125, 1.275) {};
		\node [style=none] (154) at (4.125, 0.525) {};
		\node [style=none] (155) at (3.375, 0.525) {};
		\node [style=none] (156) at (2.975, 0.375) {};
		\node [style=none] (157) at (3.625, 0.375) {};
		\node [style=none] (158) at (3.625, -0.375) {};
		\node [style=none] (159) at (2.875, -0.375) {};
		\node [style=none] (160) at (2.025, 0.375) {};
		\node [style=none] (161) at (2.775, 0.375) {};
		\node [style=none] (162) at (2.775, -0.375) {};
		\node [style=none] (163) at (2.025, -0.375) {};
		\node [style=none] (164) at (3.375, -0.525) {};
		\node [style=none] (165) at (4.125, -0.525) {};
		\node [style=none] (166) at (4.125, -1.275) {};
		\node [style=none] (167) at (3.375, -1.275) {};
		\node [style=none] (168) at (2.45, -0.525) {};
		\node [style=none] (169) at (3.2, -0.525) {};
		\node [style=none] (170) at (3.2, -1.275) {};
		\node [style=none] (171) at (2.45, -1.275) {};
		\node [style=none] (172) at (1.525, -0.525) {};
		\node [style=none] (173) at (2.275, -0.525) {};
		\node [style=none] (174) at (2.275, -1.275) {};
		\node [style=none] (175) at (1.525, -1.275) {};
		\node [style=none] (176) at (-3.425, 2.2) {};
		\node [style=none] (177) at (-2.425, 2.2) {};
		\node [style=none] (178) at (-2.425, 1.45) {};
		\node [style=none] (179) at (-3.425, 1.45) {};
		\node [style=none] (180) at (-0.475, 2.2) {};
		\node [style=none] (181) at (0.525, 2.2) {};
		\node [style=none] (182) at (0.525, 1.45) {};
		\node [style=none] (183) at (-0.475, 1.45) {};
		\node [style=none] (184) at (2.325, 2.2) {};
		\node [style=none] (185) at (3.325, 2.2) {};
		\node [style=none] (186) at (3.325, 1.45) {};
		\node [style=none] (187) at (2.325, 1.45) {};
	\end{pgfonlayer}
	\begin{pgfonlayer}{edgelayer}
		\draw (6) to (0);
		\draw (0) to (7);
		\draw (0) to (8);
		\draw (1) to (6);
		\draw (1) to (7);
		\draw (1) to (8);
		\draw (0) to (3);
		\draw (0) to (2);
		\draw (1) to (3);
		\draw (1) to (4);
		\draw [style=undirected red] (5) to (4);
		\draw [style=undirected red] (5) to (3);
		\draw [style=undirected red] (5) to (2);
		\draw [style=undirected red] (6) to (9);
		\draw [style=undirected red, in=90, out=-90] (7) to (9);
		\draw [style=undirected red] (9) to (8);
		\draw [style=undirected dashed] (30.center) to (31.center);
		\draw [style=undirected dashed] (31.center) to (32.center);
		\draw [style=undirected dashed] (32.center) to (33.center);
		\draw [style=undirected dashed] (33.center) to (30.center);
		\draw [style=undirected dashed] (24.center) to (27.center);
		\draw [style=undirected dashed] (27.center) to (29.center);
		\draw [style=undirected dashed] (29.center) to (28.center);
		\draw [style=undirected dashed] (28.center) to (24.center);
		\draw [style=undirected dashed] (34.center) to (35.center);
		\draw [style=undirected dashed] (35.center) to (36.center);
		\draw [style=undirected dashed] (34.center) to (37.center);
		\draw [style=undirected dashed] (37.center) to (36.center);
		\draw [style=undirected dashed] (38.center) to (41.center);
		\draw [style=undirected dashed] (41.center) to (40.center);
		\draw [style=undirected dashed] (40.center) to (39.center);
		\draw [style=undirected dashed] (39.center) to (38.center);
		\draw (2) to (3);
		\draw (7) to (8);
		\draw (87) to (81);
		\draw (81) to (88);
		\draw (81) to (89);
		\draw (82) to (87);
		\draw (82) to (88);
		\draw (82) to (89);
		\draw (81) to (84);
		\draw (81) to (83);
		\draw (82) to (84);
		\draw (82) to (85);
		\draw [style=undirected red] (86) to (85);
		\draw [style=undirected red] (86) to (84);
		\draw [style=undirected red] (86) to (83);
		\draw [style=undirected red] (87) to (90);
		\draw [style=undirected red, in=90, out=-90] (88) to (90);
		\draw [style=undirected red] (90) to (89);
		\draw [style=undirected dashed] (95.center) to (96.center);
		\draw [style=undirected dashed] (96.center) to (97.center);
		\draw [style=undirected dashed] (97.center) to (98.center);
		\draw [style=undirected dashed] (98.center) to (95.center);
		\draw [style=undirected dashed] (91.center) to (92.center);
		\draw [style=undirected dashed] (92.center) to (94.center);
		\draw [style=undirected dashed] (94.center) to (93.center);
		\draw [style=undirected dashed] (93.center) to (91.center);
		\draw [style=undirected dashed] (99.center) to (100.center);
		\draw [style=undirected dashed] (100.center) to (101.center);
		\draw [style=undirected dashed] (99.center) to (102.center);
		\draw [style=undirected dashed] (102.center) to (101.center);
		\draw [style=undirected dashed] (103.center) to (106.center);
		\draw [style=undirected dashed] (106.center) to (105.center);
		\draw [style=undirected dashed] (105.center) to (104.center);
		\draw [style=undirected dashed] (104.center) to (103.center);
		\draw (83) to (84);
		\draw (88) to (89);
		\draw (115) to (109);
		\draw (109) to (116);
		\draw [style=undirected] (109) to (117);
		\draw (110) to (115);
		\draw (110) to (116);
		\draw (110) to (117);
		\draw (109) to (112);
		\draw (109) to (111);
		\draw (110) to (112);
		\draw (110) to (113);
		\draw [style=undirected red] (114) to (113);
		\draw [style=undirected red] (114) to (112);
		\draw [style=undirected red] (114) to (111);
		\draw [style=undirected red] (115) to (118);
		\draw [style=undirected red, in=90, out=-90] (116) to (118);
		\draw [style=undirected red] (118) to (117);
		\draw [style=undirected dashed] (131.center) to (134.center);
		\draw [style=undirected dashed] (134.center) to (133.center);
		\draw [style=undirected dashed] (133.center) to (132.center);
		\draw [style=undirected dashed] (132.center) to (131.center);
		\draw (111) to (112);
		\draw (116) to (117);
		\draw [style=undirected dashed] (136.center) to (137.center);
		\draw [style=undirected dashed] (137.center) to (138.center);
		\draw [style=undirected dashed] (138.center) to (139.center);
		\draw [style=undirected dashed] (139.center) to (136.center);
		\draw [style=undirected dashed] (140.center) to (141.center);
		\draw [style=undirected dashed] (141.center) to (142.center);
		\draw [style=undirected dashed] (142.center) to (143.center);
		\draw [style=undirected dashed] (143.center) to (140.center);
		\draw [style=undirected dashed] (144.center) to (145.center);
		\draw [style=undirected dashed] (145.center) to (146.center);
		\draw [style=undirected dashed] (146.center) to (147.center);
		\draw [style=undirected dashed] (147.center) to (144.center);
		\draw [style=undirected dashed] (148.center) to (149.center);
		\draw [style=undirected dashed] (149.center) to (150.center);
		\draw [style=undirected dashed] (150.center) to (151.center);
		\draw [style=undirected dashed] (151.center) to (148.center);
		\draw [style=undirected dashed] (152.center) to (153.center);
		\draw [style=undirected dashed] (153.center) to (154.center);
		\draw [style=undirected dashed] (154.center) to (155.center);
		\draw [style=undirected dashed] (155.center) to (152.center);
		\draw [style=undirected dashed] (156.center) to (157.center);
		\draw [style=undirected dashed] (157.center) to (158.center);
		\draw [style=undirected dashed] (158.center) to (159.center);
		\draw [style=undirected dashed] (159.center) to (156.center);
		\draw [style=undirected dashed] (160.center) to (161.center);
		\draw [style=undirected dashed] (161.center) to (162.center);
		\draw [style=undirected dashed] (162.center) to (163.center);
		\draw [style=undirected dashed] (163.center) to (160.center);
		\draw [style=undirected dashed] (164.center) to (165.center);
		\draw [style=undirected dashed] (165.center) to (166.center);
		\draw [style=undirected dashed] (166.center) to (167.center);
		\draw [style=undirected dashed] (167.center) to (164.center);
		\draw [style=undirected dashed] (168.center) to (169.center);
		\draw [style=undirected dashed] (169.center) to (170.center);
		\draw [style=undirected dashed] (170.center) to (171.center);
		\draw [style=undirected dashed] (171.center) to (168.center);
		\draw [style=undirected dashed] (172.center) to (173.center);
		\draw [style=undirected dashed] (173.center) to (174.center);
		\draw [style=undirected dashed] (174.center) to (175.center);
		\draw [style=undirected dashed] (175.center) to (172.center);
		\draw [style=undirected dashed] (176.center) to (179.center);
		\draw [style=undirected dashed] (179.center) to (178.center);
		\draw [style=undirected dashed] (178.center) to (177.center);
		\draw [style=undirected dashed] (177.center) to (176.center);
		\draw [style=undirected dashed] (180.center) to (183.center);
		\draw [style=undirected dashed] (183.center) to (182.center);
		\draw [style=undirected dashed] (182.center) to (181.center);
		\draw [style=undirected dashed] (181.center) to (180.center);
		\draw [style=undirected dashed] (184.center) to (187.center);
		\draw [style=undirected dashed] (187.center) to (186.center);
		\draw [style=undirected dashed] (186.center) to (185.center);
		\draw [style=undirected dashed] (185.center) to (184.center);
	\end{pgfonlayer}
\end{tikzpicture}
    \caption{Influence of $\alpha$: As $\alpha$ increases, the criterion for clustering nodes by path similarity becomes stricter. The source node, $\texttt{B}_1$, is shaded in grey.
Nodes that were previously clustered become partitioned. 
For example: $\{P_1, P_2, P_3\} \xrightarrow{} \{P_1, P_2\}\{P_3\} 
\xrightarrow{} \{P_1\} \{P_2\} \{P_3\}$.}
    \label{fig:alpha}
\end{figure}

\textbf{Motivation} Theorem~\ref{theorem:optimality} allows us to run the minimum number of walks needed to compute good enough estimates of the truncated hitting times and the path distributions. Based on these estimates, the next step is to partition our data into path-symmetric sets. If we choose very strict standards for clustering nodes together, i.e. only clustering if they have identical truncated hitting times and path distributions, the formulae we find will be very stringent and will not generalize well on unseen data (overfitting). However, if we are too loose on the criteria to merge nodes, the rules we obtain will be too soft and not informative enough (underfitting). 

Our approach to controlling rule softness is to introduce statistical tests to decide when two nodes are distance- and path-symmetric and a user-specified parameter $0 < \alpha < 1$. $\alpha$ is the statistical significance level at which two nodes are considered distance- and path-symmetric. $\alpha$, therefore, measures how lenient we are in merging entities into abstract concepts and by extension an indirect measure of the softness of rules. The effect of changing $\alpha$ for path-symmetry clustering on an example hypergraph is shown in Fig \ref{fig:alpha}.

This approach results in three major benefits:

(i) $\alpha$ is the only clustering parameter compared to the four parameters in SOTA. (ii) $\alpha$ is by construction dataset independent, thus simplifying hyperparameter tuning compared to SOTA. (iii) the $\alpha$ parameter has a direct and intuitive effect on the size of the path-symmetric clusters, with smaller $\alpha$ leading to less-strict statistical tests that ultimately favour finding fewer, but larger path-symmetric sets and thus fewer, more-approximate abstract concepts.

\textbf{Usage} Given truncated hitting times $\hat{h}^{L,N}_{i,j}$, we merge nodes if the difference between their THTs is below a threshold $\theta_{sym}(\alpha)$. Next, given path distributions $\hat{P}_{i,j}^{L,N}$, we propose a hypothesis test to validate whether a set of sampled distributions are statistically similar. We show that both tests can be performed to a specified level of statistical significance given by just one parameter: $\alpha$.

First, we consider the null hypothesis that nodes $v_j$ and $v_k$ are order-$L$ distance-symmetric w.r.t. $v_i$, using  ${\vert \hat{h}^{L,N}_{i,j} - \hat{h}^{L,N}_{i,k} \vert}$ as a test statistic:

\begin{theorem}[Distance-Symmetric Hypothesis Test]
\label{theta_sym_hypothesis}
The null hypothesis is rejected at significance level $\alpha$, i.e. nodes $v_j$ and $v_k$ are not order-$L$ distance-symmetric, if 
$\vert \hat{h}^{L,N}_{i,j} - \hat{h}^{L,N}_{i,k} \vert > (({L-1})/{\sqrt{2N}}) t_{\alpha/2, N-1}$,
where $t_{\alpha/2, N-1}$ is the inverse-survival function of an $N-1$ degrees of freedom student-t distribution evaluated at $\alpha/2$. 
\end{theorem}

Theorem~\ref{theta_sym_hypothesis} allows us to set parameter $\sym$ dynamically for each pair of nodes whose hitting times are being compared, such that nodes are merged only if they are distance-symmetric at significance level $\alpha$:
\begin{equation}
    \sym = \frac{L-1}{\sqrt{2N}}t_{\alpha/2, N-1}. \label{eq:theta_sym_hypothesis}
\end{equation}

To measure the degree of \textit{path} symmetry (as opposed to \textit{distance} symmetry), a threshold can be set using a different hypothesis test but based on the same $\alpha$ used to set $\sym$ above. In the next section, we detail this hypothesis test.

\section{Efficient Structural Motif Finding}\label{section:ours:path-symmetry}

Above we have discussed how to set our parameters in a principled fashion. In this section, we discuss how to use these parameters in an efficient algorithm (Sec. \ref{subsec:pathsymmetryclustering}), and then further improve speed by reducing the required length (Sec. \ref{section:h-clustering}) and number of random walks  (Sec. \ref{subsec:runningfewerwalks}).
\subsection{An Improved Path-Symmetry Clustering Algorithm}
\label{subsec:pathsymmetryclustering}

In this section, we outline an efficient algorithm, which we refer to as $\ourspaths$, for partitioning nodes into sets that are path-symmetric at significance level $\alpha$. This algorithm has $\mathcal{O}(n\ln{n})$ complexity in the number of nodes to cluster, which offers a significant improvement over the $\mathcal{O}(n^3)$ complexity of SOTA.

Using the notation introduced in Section~\ref{section:ours:parameters}, we partition each distance-symmetric node set $A_m \in \{A_1,\dots, A_M\}$ into 
path-symmetric sets w.r.t. a node $v_i$. $\ourspaths$ treats the path counts of the nodes within each $A_m$ as points in a multi-dimensional space of the path signatures. For each $A_m$, $\ourspaths$ then clusters nodes into path-symmetric sets as follows: First, we run a hypothesis test (Thm. \ref{hypothesis_test}, discussed below) on $A_m$ to check whether the entire set of nodes is path-symmetric at significance level $\alpha$. If the test passes, all nodes are clustered together. If the test fails, we proceed with recursive clustering (see Alg. \ref{k_means_clustering}):
\begin{enumerate}[leftmargin=*]
    \item Standardize (zero mean, unit variance) the path counts and use PCA to map these standardized counts to a two-dimensional space.
    \item Cluster nodes in the reduced space into two sets using unsupervised BIRCH clustering \cite{zhang_birch_1996} .
    \item Perform a path-symmetry hypothesis test (Thm. \ref{hypothesis_test}) separately on the two identified sets.
    \item Clusters failing the test have their nodes repeatedly repartitioned into two new clusters using steps 2 and 3 until all sets of clusters pass the hypothesis test. The output is a set of clusters $\{B_1, B_2, ..., B_k\}$ partitioning the set $A_m$.
\end{enumerate}

\begin{algorithm}[t] 
    \caption{$\ourspaths$} \label{k_means_clustering}
    \textbf{Input:}{$A$, nodes to partition into order-$L$ path-symmetric sets w.r.t. $v_i$, where $v_i \not \in A$}\\
    \textbf{Output:}{$B_1,\dots,B_K$, path-symmetric sets}\\
    \textbf{Parameters:}{$\alpha$, $\ourdim=2$ (parameters $N,L$ are implicit)} 

    \uIf{$A$ is path symmetric at significance level $\alpha$ for each  
    $l \in \{L, L-1, \dots, 1\}$} 
        {\Return{$\left\{A\right\}$} \tcp*{Thm.~\ref{hypothesis_test}} 
    } 
    \Else{
        \textbf{for each} $v_{\ell'} \in A$
        \textbf{compute} \& \textbf{standardise} $\hat{C}^{L,N}_{i,\ell'}$ 
        
        \textbf{reduce} all the $\hat{C}^{l,N}_{i,\ell'}$'s into $\ourdim$-dimensional feature vectors using PCA 
        
        $Partition \xleftarrow{} \emptyset$
        
        $RemainingSets \xleftarrow{} \{A\}$
        
        \While{ $RemainingSets$ not \textbf{empty}} {
        
            $S \xleftarrow{} RemainingSets$.pop
            
            \textbf{partition} $S$ into $\{B_1, B_2\}$ via unsupervised clustering of the $\hat{C}^{l,N}_{i,\ell'}$'s  
            
            \For{$B_i \in \{B_1, B_2\}$} {
            
            \uIf{$B_i$ is path symmetric at significance level $\alpha$ for each  
    $l \in \{L, L-1, \dots, 1\}$}{
    $Parition$.append($B_i$)}

            \Else{
            $RemainingSets$.append($B_i$)
            }
        }
        }
        
        \Return{$Partition$} 
    }
\end{algorithm}

Similar to Section~\ref{section:ours:significance}, the partitioning process in $\ourspaths$ is driven by a statistical hypothesis test. This time, the desired null hypothesis is that \textit{all the nodes} in a cluster $B_k$ are order-$L$ path-symmetric. Specifically, we test for each cluster the following: that for each $l \in \{L, L-1, \dots, 1\}$ the cluster $B_k$ is \textit{exact} order-$l$ path-symmetric w.r.t. $v_i$ at significance level $\alpha$. We denote this null hypothesis $H_0$.

If $H_0$ is true, then at significance level $\alpha$ there exists a multinomial distribution, common to all nodes in $B_k$, from which the empirical exact path counts $\hat{C}_{i,j}^{L,N}|_l$ are drawn.
Extending a version of the $\chi^2$ test, we show the following:

\begin{theorem}[Path-Symmetric Hypothesis Test]
\label{hypothesis_test}
Let $\Lambda_l$ be the total number of different paths of length $l$ over all $\mathcal{S}^L_{i,j}$'s. The null hypothesis that the nodes $B_k$ are order-$l$ exact path symmetric is rejected at significance level $\alpha$ if the statistic 
\begin{align}
    Q(B_k)  \vcentcolon= \sum_{\lambda=0}^{\Lambda_l}\sum_{v_{j}\in B_k} \left(c_\lambda - c^{(j)}_\lambda\right)^2
\end{align}
exceeds $\chi_{\boldsymbol{w},\boldsymbol{\nu}}^2(\alpha)$, where
\begin{align}
    c_\lambda^{(j)} \vcentcolon= \hat{C}_{i,j}^{L,N}(\lambda), \quad
    c^{(j)}_0 \vcentcolon= N - \sum_{\lambda=1}^{\Lambda_l} c^{(j)}_{\lambda},\\ \quad
    c_\lambda \vcentcolon = \frac{1}{\vert B_k \vert}\sum_{v_{j} \in B_k} c_\lambda^{(j)},    \nonumber
\end{align}
and $\chi^2_{\boldsymbol{w}, \boldsymbol{\nu}}(\alpha)$ is a generalised chi-squared distribution, weight parameters $\boldsymbol{w}$ and degree of freedom parameters $\boldsymbol{\nu} = \left(1, 1, ..., 1\right)$, evaluated at significance level $\alpha$. Above $\boldsymbol{w}$ are the eigenvalues of the block matrix $\boldsymbol{\Tilde{\Sigma}}$ with components
\begin{align*}
    &\Tilde{\Sigma}^{(b,b')}_{\lambda,\lambda'} = N\left(\delta_{b,b'} - \frac{1}{\vert B_k \vert} \right) \cdot\\ &\Big(\delta_{\lambda,\lambda'}\frac{c_\lambda}{N}\left(1-\frac{c_\lambda}{N}\right) - \big(1-\delta_{\lambda,\lambda'}\big)\frac{c_\lambda}{N}\frac{c_\lambda'}{N}\Big) \;,
\end{align*}
where $b,b' \in \left(1, \dots, \vert B_k \vert \right)$ index blocks, $\lambda,\lambda' \in \left(0, \dots, \Lambda_l \right)$ index within blocks and $\delta_{\lambda,\lambda'}$ is the Kronecker delta.
\end{theorem}

\begin{remark}
By requiring knowledge of the eigenvalues $\boldsymbol{w}$, Theorem \ref{hypothesis_test} suggests that an eigendecomposition of $\boldsymbol{\Tilde{\Sigma}}$ is necessary. However, we show in the Appendix how this calculation can be avoided by approximating the generalised chi-squared distribution by a gamma distribution.
\end{remark}

\subsection{Running Shorter Random Walks}
\label{section:h-clustering}
We now show how to set a minimal $L$ needed for the random walks to span areas of interest in the hypergraph. We do this by hierarchical clustering. 

\textbf{Motivation}  We implement hierarchical clustering by iteratively cutting the hypergraph along sparse cuts until no sparse cuts remain. This algorithm results in three benefits:
(i) Splitting the hypergraph into smaller sub-graphs leads to smaller diameters and therefore smaller $L$. This, by extension, also reduces $N$ (Thm. \ref{theorem:optimality}).
(ii) Having fewer nearby nodes means that the subsequent partitioning in $\ourspaths$ is faster.
(iii) Hierarchical clustering identifies groups of densely-connected nodes, which helps us to ignore spurious links. Spuriously-connected nodes appear rarely in the path signatures and therefore only add noise to the path signature counts. By focusing random walks on areas of interest, we are hitting nodes that are densely connected more often and gaining more accurate statistics of truncated hitting times and empirical path distributions.

\textbf{Hierarchical Clustering Algorithm} The algorithm $\hcluster$ is based on spectral clustering, a standard approach for cutting graphs along sparse cuts. A discussion of spectral clustering is beyond the scope of this paper. Note that the is no equivalent approach for hypergraphs, so we propose to translate a hypergraph into a graph and then perform spectral clustering as follows:

In overview, $\hcluster$ begins by converting a hypergraph $\mathcal{H}=(V,E)$ into a weighted graph $\mathcal{G}$ by expanding cliques over each hyperedge. Next, $\mathcal{G}$ is recursively bipartitioned using the sweep set approximation algorithm for the Cheeger-cut \cite{chang_cheegers_2017}. The result of the partitioning is a set of subgraphs $G\vcentcolon=\{\mathcal{G}_1, \mathcal{G}_2, ..., \mathcal{G}_k\}$. The partitioning terminates whenever the second-smallest eigenvalue of the symmetric Laplacian matrix $\lambda_2$ exceeds a threshold value $\lambda_2^{max}$.
$\lambda_2^{max}$ is dataset independent and thus fixed in our implementation.
Finally, each subgraph $\mathcal{G}_i$ is then converted into a hypergraph $\mathcal{H}_i = (V_i, E_i)$ such that the vertex set $V_i$ of the hypergraph is initialised to be the vertex set of $\mathcal{G}_i$. The edge set $E_i$ is then constructed by adding all hyperedges $e \in E$ whose strict majority of element vertices appear in $V_i$, i.e. $E_i \vcentcolon= \left\{e \in E \; \middle\vert \; \vert e \cap V_i \vert > \vert e\vert/2 \right\}$. As a consequence, no nodes nor edges are lost during clustering.
$\hcluster$ returns the set of sub-hypergraphs $\{\mathcal{H}_1, \mathcal{H}_2, ..., \mathcal{H}_k\}$. After partitioning, we run the rest of the pipeline with $L$ set to the diameter of each $\mathcal{H}_i$.

Our entire pipeline for learning abstract concepts from a relational database is summarised in Algorithm \ref{get_communities_improved}.

%%%%%%%%%%%%%% ALGORITHM
\begin{algorithm}[tb]
    \caption{$\oursgetcomm$}\label{get_communities_improved}
    \textbf{Input:}{$\mathcal{H}$, the hypergraph representation of the input relational database}\\
    \textbf{Output:}{Path-symmetric sets (abstract concepts $\mathcal{C}$) of nodes w.r.t. each $v_i$ in $\mathcal{H}$} \\
    \textbf{Parameters:}{$\varepsilon, \alpha$}
    
    $\mathcal{H}_1, \dots, \mathcal{H}_K := \hcluster(\mathcal{H})$ \tcp*{Sec. \ref{section:h-clustering}} 
    
    Let $V_k$ denote the set of nodes in $\mathcal{H}_k$
    
    \For{$1 \leq  k \leq K$}{
        
        \textbf{set} $L$ to the diameter of $\mathcal{H}_k$
        
        \textbf{compute} $\varepsilon$-optimal $N$ on $\mathcal{H}_k$ under $L$ \tcp*{Th.\ref{theorem:optimality}}
        
        \For{\textbf{each} node $v_i$ in $\mathcal{H}_k$}{
        
        \textbf{for each} $v_j \neq v_i$ in $\mathcal{H}_k$ compute $\hat{P}^{L,N}_{i,j}$ and $\hat{h}^{L,N}_{i,j}$ \tcp*{Sec. \ref{section:preliminaries}}

        \textbf{partition} $V_k$ into 
        distance-symmetric sets $\{A_1, A_2, ..., A_M\}$ using $\alpha$-significance \tcp*{Th. \ref{theta_sym_hypothesis}, Sec. \ref{section:preliminaries}}
    
        \For{ $1 \leq m \leq M$ }{
               $\mathcal{C}_m := \ourspaths(A_m, \alpha)$ \tcp*{Th.~\ref{hypothesis_test}, Sec. \ref{section:ours:path-symmetry}}
        
        }
        }
    }
    \Return{all $\mathcal{C}_m$'s}
\end{algorithm}

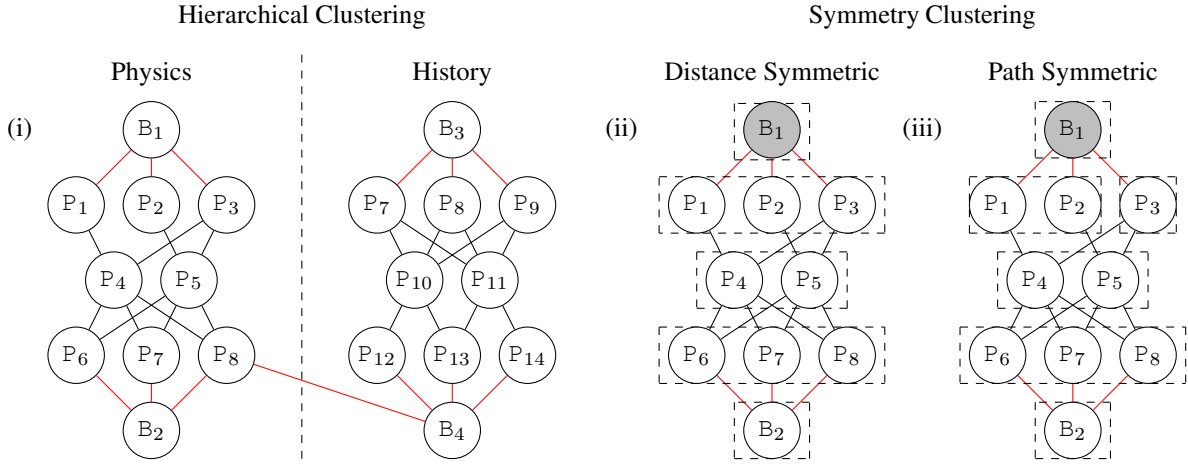
\begin{figure*}[ht]
\centering
\begin{tikzpicture}
	\begin{pgfonlayer}{nodelayer}
		\node [style=mwc] (0) at (1.5, 0) {$\texttt{P}_4$};
		\node [style=mwc] (1) at (2.5, 0) {$\texttt{P}_5$};
		\node [style=mwc] (2) at (1, 1) {$\texttt{P}_1$};
		\node [style=mwc] (3) at (2, 1) {$\texttt{P}_2$};
		\node [style=mwc] (4) at (3, 1) {$\texttt{P}_3$};
		\node [style=mbc] (5) at (2, 2) {$\texttt{B}_1$};
		\node [style=mwc] (6) at (1, -1) {$\texttt{P}_6$};
		\node [style=mwc] (7) at (2, -1) {$\texttt{P}_7$};
		\node [style=mwc] (8) at (3, -1) {$\texttt{P}_8$};
		\node [style=mwc] (9) at (2, -2) {$\texttt{B}_2$};
		\node [style=none] (23) at (2, 2.75) {Distance Symmetric};
		\node [style=none] (24) at (1, 0.375) {};
		\node [style=none] (27) at (1, -0.375) {};
		\node [style=none] (28) at (3, 0.375) {};
		\node [style=none] (29) at (3, -0.375) {};
		\node [style=none] (30) at (0.5, 1.375) {};
		\node [style=none] (31) at (3.5, 1.375) {};
		\node [style=none] (32) at (3.5, 0.625) {};
		\node [style=none] (33) at (0.5, 0.625) {};
		\node [style=none] (34) at (0.5, -0.625) {};
		\node [style=none] (35) at (0.5, -1.375) {};
		\node [style=none] (36) at (3.5, -1.375) {};
		\node [style=none] (37) at (3.5, -0.625) {};
		\node [style=none] (38) at (1.5, -1.625) {};
		\node [style=none] (39) at (2.5, -1.625) {};
		\node [style=none] (40) at (2.5, -2.375) {};
		\node [style=none] (41) at (1.5, -2.375) {};
		\node [style=mwc] (42) at (5.5, 0) {$\texttt{P}_4$};
		\node [style=mwc] (43) at (6.5, 0) {$\texttt{P}_5$};
		\node [style=mwc] (44) at (5, 1) {$\texttt{P}_1$};
		\node [style=mwc] (45) at (6, 1) {$\texttt{P}_2$};
		\node [style=mwc] (46) at (7, 1) {$\texttt{P}_3$};
		\node [style=mbc] (47) at (6, 2) {$\texttt{B}_1$};
		\node [style=mwc] (48) at (5, -1) {$\texttt{P}_6$};
		\node [style=mwc] (49) at (6, -1) {$\texttt{P}_7$};
		\node [style=mwc] (50) at (7, -1) {$\texttt{P}_8$};
		\node [style=mwc] (51) at (6, -2) {$\texttt{B}_2$};
		\node [style=none] (52) at (6, 2.75) {Path Symmetric};
		\node [style=none] (53) at (5, 0.375) {};
		\node [style=none] (54) at (5, -0.375) {};
		\node [style=none] (55) at (7, 0.375) {};
		\node [style=none] (56) at (7, -0.375) {};
		\node [style=none] (57) at (4.625, 1.375) {};
		\node [style=none] (58) at (7.375, 1.375) {};
		\node [style=none] (59) at (7.375, 0.625) {};
		\node [style=none] (60) at (4.625, 0.625) {};
		\node [style=none] (61) at (4.5, -0.625) {};
		\node [style=none] (62) at (4.5, -1.375) {};
		\node [style=none] (63) at (7.5, -1.375) {};
		\node [style=none] (64) at (7.5, -0.625) {};
		\node [style=none] (65) at (5.5, -1.625) {};
		\node [style=none] (66) at (6.5, -1.625) {};
		\node [style=none] (67) at (6.5, -2.375) {};
		\node [style=none] (68) at (5.5, -2.375) {};
		\node [style=none] (70) at (6.625, 1.375) {};
		\node [style=none] (72) at (6.625, 0.625) {};
		\node [style=none] (74) at (6.375, 1.375) {};
		\node [style=none] (76) at (6.375, 0.625) {};
		\node [style=none] (79) at (0, 2) {(ii)};
		\node [style=none] (80) at (4, 2) {(iii)};
		\node [style=mwc] (81) at (-6.75, 0) {$\texttt{P}_4$};
		\node [style=mwc] (82) at (-5.75, 0) {$\texttt{P}_5$};
		\node [style=mwc] (83) at (-7.25, 1) {$\texttt{P}_1$};
		\node [style=mwc] (84) at (-6.25, 1) {$\texttt{P}_2$};
		\node [style=mwc] (85) at (-5.25, 1) {$\texttt{P}_3$};
		\node [style=mwc] (86) at (-6.25, 2) {$\texttt{B}_1$};
		\node [style=mwc] (87) at (-7.25, -1) {$\texttt{P}_6$};
		\node [style=mwc] (88) at (-6.25, -1) {$\texttt{P}_7$};
		\node [style=mwc] (89) at (-5.25, -1) {$\texttt{P}_8$};
		\node [style=mwc] (90) at (-6.25, -2) {$\texttt{B}_2$};
		\node [style=mwc] (91) at (-3.25, -1) {$\texttt{P}_{12}$};
		\node [style=mwc] (92) at (-2.25, -1) {$\texttt{P}_{13}$};
		\node [style=mwc] (93) at (-1.25, -1) {$\texttt{P}_{14}$};
		\node [style=mwc] (94) at (-2.25, -2) {$\texttt{B}_4$};
		\node [style=mwc] (95) at (-2.75, 0) {$\texttt{P}_{10}$};
		\node [style=mwc] (96) at (-1.75, 0) {$\texttt{P}_{11}$};
		\node [style=mwc] (97) at (-3.25, 1) {$\texttt{P}_7$};
		\node [style=mwc] (98) at (-2.25, 1) {$\texttt{P}_8$};
		\node [style=mwc] (99) at (-1.25, 1) {$\texttt{P}_9$};
		\node [style=mwc] (100) at (-2.25, 2) {$\texttt{B}_3$};
		\node [style=none] (101) at (-4.25, 3) {};
		\node [style=none] (102) at (-4.25, -2.5) {};
		\node [style=none] (104) at (-6.25, 2.75) {Physics};
		\node [style=none] (105) at (-2.25, 2.75) {History};
		\node [style=none] (106) at (-8, 2) {(i)};
		\node [style=none] (107) at (-4.25, 3.5) {Hierarchical Clustering};
		\node [style=none] (108) at (5.5, 2.375) {};
		\node [style=none] (109) at (6.5, 2.375) {};
		\node [style=none] (110) at (6.5, 1.625) {};
		\node [style=none] (111) at (5.5, 1.625) {};
		\node [style=none] (112) at (1.5, 2.35) {};
		\node [style=none] (113) at (2.5, 2.35) {};
		\node [style=none] (114) at (2.5, 1.6) {};
		\node [style=none] (115) at (1.5, 1.6) {};
		\node [style=none] (116) at (4, 3.5) {Symmetry Clustering};
	\end{pgfonlayer}
	\begin{pgfonlayer}{edgelayer}
		\draw (6) to (0);
		\draw (0) to (7);
		\draw (0) to (8);
		\draw (1) to (6);
		\draw (1) to (7);
		\draw (1) to (8);
		\draw (0) to (2);
		\draw (1) to (3);
		\draw (1) to (4);
		\draw [style=undirected red] (5) to (4);
		\draw [style=undirected red] (5) to (3);
		\draw [style=undirected red] (5) to (2);
		\draw [style=undirected red] (6) to (9);
		\draw [style=undirected red, in=90, out=-90] (7) to (9);
		\draw [style=undirected red] (9) to (8);
		\draw [style=blackdashed] (30.center) to (31.center);
		\draw [style=blackdashed] (31.center) to (32.center);
		\draw [style=blackdashed] (32.center) to (33.center);
		\draw [style=blackdashed] (33.center) to (30.center);
		\draw [style=blackdashed] (24.center) to (27.center);
		\draw [style=blackdashed] (27.center) to (29.center);
		\draw [style=blackdashed] (29.center) to (28.center);
		\draw [style=blackdashed] (28.center) to (24.center);
		\draw [style=blackdashed] (34.center) to (35.center);
		\draw [style=blackdashed] (35.center) to (36.center);
		\draw [style=blackdashed] (34.center) to (37.center);
		\draw [style=blackdashed] (37.center) to (36.center);
		\draw [style=blackdashed] (38.center) to (41.center);
		\draw [style=blackdashed] (41.center) to (40.center);
		\draw [style=blackdashed] (40.center) to (39.center);
		\draw [style=blackdashed] (39.center) to (38.center);
		\draw (48) to (42);
		\draw (42) to (49);
		\draw (42) to (50);
		\draw (43) to (48);
		\draw (43) to (49);
		\draw (43) to (50);
		\draw (42) to (44);
		\draw (43) to (45);
		\draw (43) to (46);
		\draw [style=undirected red] (47) to (46);
		\draw [style=undirected red] (47) to (45);
		\draw [style=undirected red] (47) to (44);
		\draw [style=undirected red] (48) to (51);
		\draw [style=undirected red, in=90, out=-90] (49) to (51);
		\draw [style=undirected red] (51) to (50);
		\draw [style=blackdashed] (58.center) to (59.center);
		\draw [style=blackdashed] (60.center) to (57.center);
		\draw [style=blackdashed] (53.center) to (54.center);
		\draw [style=blackdashed] (54.center) to (56.center);
		\draw [style=blackdashed] (56.center) to (55.center);
		\draw [style=blackdashed] (55.center) to (53.center);
		\draw [style=blackdashed] (61.center) to (62.center);
		\draw [style=blackdashed] (62.center) to (63.center);
		\draw [style=blackdashed] (61.center) to (64.center);
		\draw [style=blackdashed] (64.center) to (63.center);
		\draw [style=blackdashed] (65.center) to (68.center);
		\draw [style=blackdashed] (68.center) to (67.center);
		\draw [style=blackdashed] (67.center) to (66.center);
		\draw [style=blackdashed] (66.center) to (65.center);
		\draw [style=blackdashed] (72.center) to (59.center);
		\draw [style=blackdashed] (72.center) to (70.center);
		\draw [style=blackdashed] (70.center) to (58.center);
		\draw [style=blackdashed] (76.center) to (74.center);
		\draw (87) to (81);
		\draw (81) to (88);
		\draw (81) to (89);
		\draw (82) to (87);
		\draw (82) to (88);
		\draw (82) to (89);
		\draw (81) to (83);
		\draw (82) to (84);
		\draw (82) to (85);
		\draw (95) to (91);
		\draw (95) to (92);
		\draw (96) to (92);
		\draw (96) to (93);
		\draw (97) to (95);
		\draw (95) to (98);
		\draw (95) to (99);
		\draw (99) to (96);
		\draw (96) to (97);
		\draw [style=undirected red] (97) to (100);
		\draw [style=undirected red] (98) to (100);
		\draw [style=undirected red] (99) to (100);
		\draw [style=undirected red] (86) to (85);
		\draw [style=undirected red] (86) to (84);
		\draw [style=undirected red] (86) to (83);
		\draw [style=undirected red] (87) to (90);
		\draw [style=undirected red, in=90, out=-90] (88) to (90);
		\draw [style=undirected red] (90) to (89);
		\draw [style=undirected red] (91) to (94);
		\draw [style=undirected red] (92) to (94);
		\draw [style=undirected red] (94) to (93);
		\draw [style=blackdashed] (101.center) to (102.center);
		\draw [style=undirected red] (89) to (94);
		\draw [style=blackdashed] (108.center) to (111.center);
		\draw [style=blackdashed] (111.center) to (110.center);
		\draw [style=blackdashed] (110.center) to (109.center);
		\draw [style=blackdashed] (109.center) to (108.center);
		\draw [style=blackdashed] (112.center) to (115.center);
		\draw [style=blackdashed] (115.center) to (114.center);
		\draw [style=blackdashed] (114.center) to (113.center);
		\draw [style=blackdashed] (113.center) to (112.center);
		\draw (85) to (81);
		\draw (0) to (4);
		\draw (42) to (46);
		\draw [style=blackdashed] (60.center) to (76.center);
		\draw [style=blackdashed] (57.center) to (74.center);
		\draw (98) to (96);
	\end{pgfonlayer}
\end{tikzpicture}
\caption{PRISM Pipeline: A visual example of the PRISM algorithm applied to an academic departments toy dataset. Nodes $\texttt{P}_i$ are entities of type \texttt{person}, while $\texttt{B}_i$ are entities of type $\texttt{book}$. Black edges represent $\textsc{teaches}(\texttt{person},\texttt{person})$, and red edges represent $\textsc{reads}(\texttt{person},\texttt{book})$  predicates. Although not explicitly annotated in the data, entities $\{\texttt{P}_4, \texttt{P}_5, \texttt{P}_{10}, \texttt{P}_{11}\}$ are in fact professors and the remaining \texttt{person} entities are students. \textit{Hierarchical Clustering Preprocessing}: The physics and history departments are connected by a single spurious link. In this example, hierarchical clustering therefore stops after one iteration, cutting along the departments (dotted line). To avoid information loss, the spurious link between $\texttt{P}_8$ and $\texttt{B}_4$ will be preserved in one of the clusters (Sec. \ref{section:h-clustering}). \textit{Symmetry Clustering}: Here we focus on the left sub-graph obtained from hierarchical clustering in (i). Running random walks from $\texttt{B}_1$, we show examples of the distance-symmetric and path-symmetric clusters that we obtain in (ii) and (iii), respectively. Note how $\{\texttt{P}_1, \texttt{P}_2, \texttt{P}_3\}$ in (ii) is partitioned into $\{\texttt{P}_1, \texttt{P}_2\}$ and $ \{\texttt{P}_3\}$ in (iii) since path-symmetry is more stringent than distance symmetry. 
}
\label{fig:concept_finding}
\end{figure*}

\subsection{Running Fewer Random Walks}
\label{subsec:runningfewerwalks}

As a final optimization step, we comment on how the number of random walks can be further reduced. The number of walks as implied by Theorem \ref{theorem:optimality} can be very large since $P^*$ grows exponentially with $L$. Therefore in practice, rather than running enough walks to guarantee $\varepsilon$-boundedness for all path signatures, we only run enough walks to guarantee $\varepsilon$-boundedness for the top $k$ most common path signatures.

\begin{theorem}[Fewer Random Walks]
\label{theorem:fewer_random_walks}
An upper bound on $N$ sufficient for the $k^{th}$ most probable path to have uncertainty less than or equal to $\varepsilon$ is
\[
N = \frac{\left(k+1\right)\left(\gamma + \ln P^{*}\right) - 1}{\varepsilon^2}.
\]
\end{theorem}

In our implementation, we use $k=3$ since we deem it to be the smallest value (and therefore requiring the fewest random walks) that still allows for meaningful comparison between path distributions.

\section{Extended Example}

We now illustrate the entire PRISM pipeline through an extended example, shown in Fig \ref{fig:concept_finding}. In the figure, we consider the hypergraph representation of a dataset describing a physics and history department in a university, containing two types of entities (\texttt{person} and \texttt{book}) and two relations ($\textsc{teaches}(\texttt{person}, \texttt{person})$ and $\textsc{reads}(\texttt{person}, \texttt{book})$). 

The first stage of PRISM applies hierarchical clustering to the dataset to identify densely-connected nodes. In this example, the physics and history departments are only connected by a single, spurious link, and the hierarchical clustering stops after one iteration, cutting along the departments. In addition, we can verify here that the hierarchical clustering results in an almost two-fold computational speed-up: The original hypergraph in Fig \ref{fig:concept_finding}(i) has diameter 9. If we set $\varepsilon=0.1$, then Theorem \ref{theorem:fewer_random_walks} gives an  upper bound on the $\varepsilon$-optimal number of random walks for accurate path distributions of $N = 3.0 \times 10^3$. After hierarchical clustering, the hypergraphs have diameter 4, and Theorem \ref{theorem:fewer_random_walks} for the same $\varepsilon$ gives an upper bound of $N = 1.6 \times 10^3$.

%%%%%% TABLE %%%%%%%
\begin{table*}[ht]
\centering
\begin{tabular}{@{}lllllll@{}}
\toprule
                       & Algorithm & AUC & CLL  & ACC                & MF TIME (s)          & SL TIME (s)        \\ \midrule
\multirow{3}{*}{\texttt{IMDB}}  & PRISM     & \textbf{0.141 $\pm$ 0.027}   & \textbf{-0.18 $\pm$ 0.03 } & \textbf{0.84 $\pm$ 0.02 }   & \textbf{0.086 $\pm$ 0.018} & 320 $\pm$ 40   \\
                       & LSM       & 0.12 $\pm$ 0.03   & -0.25 $\pm$ 0.06  & 0.78 $\pm$ 0.04   & 1.25 $\pm$ 0.10   & 430 $\pm$ 20   \\
                       & BOOSTR    & 0.062 $\pm$ 0.013   & -0.69 $\pm$  0.006   & 0.504 $\pm$ 0.004             & N/A                 & \textbf{165.7 $\pm$ 129 }            \\ \midrule
\multirow{3}{*}{\texttt{UWCSE}} & PRISM     & \textbf{0.402 $\pm$ 0.028}   & \textbf{-0.0098 $\pm$ 0.0009} & \textbf{0.993 $\pm$ 0.002}  & \textbf{0.40 $\pm$ 0.06}   & 640 $\pm$ 350  \\
                       & LSM       & 0.392 $\pm$ 0.023   & \textbf{-0.0098 $\pm$ 0.0009} & 0.992 $\pm$ 0.002 & 4.23 $\pm$ 0.80   & 3140 $\pm$ 270 \\
                       & BOOSTR    & 0.0098 $\pm$ 0.003  & -2.114 $\pm$ 0.004 & 0.121 $\pm$ 0.001                    & N/A                 & \textbf{30.5 $\pm$ 3.9}              \\ \midrule
\multirow{3}{*}{\texttt{WEBKB}} & PRISM     & \textbf{0.57 $\pm$ 0.04}   & \textbf{-0.0092 $\pm$ 0.0011} & \textbf{0.991 $\pm$ 0.002}                  & \textbf{0.118 $\pm$ 0.038}                 & 102 $\pm$ 5              \\
                       & LSM       & \textbf{0.57 $\pm$ 0.04}   & \textbf{-0.0092 $\pm$ 0.0011}   & \textbf{0.991 $\pm$ 0.002}                 & 2.5 $\pm$ 0.4                 & 220 $\pm$ 10              \\
                       & BOOSTR    & 0.0335 $\pm$ 0.0021   & -2.14 $\pm$ 0.09  & 0.118 $\pm$ 0.010                  & N/A                 & \textbf{9.3 $\pm$ 0.4}              \\ \bottomrule
\end{tabular}
\caption{Area Under the Precision Recall Curve (AUC), Conditional Log Likelihood (CLL), Accuracy (ACC), Motif Finding (MF) time, and Structure Learning (SL) time comparisons of \texttt{PRISM}, \texttt{LSM} and \texttt{BOOSTR} on three datasets.}
\label{tab:results_acc}
\end{table*}

\begin{figure}[t]
    \centering
    \begin{tikzpicture}
	\begin{pgfonlayer}{nodelayer}
		\node [style=mbc] (0) at (-0.5, 0) {$\texttt{P}_4$};
		\node [style=mwc] (1) at (0.5, 0) {$\texttt{P}_5$};
		\node [style=mwc] (2) at (-1, 1) {$\texttt{P}_1$};
		\node [style=mwc] (3) at (0, 1) {$\texttt{P}_2$};
		\node [style=mwc] (4) at (1, 1) {$\texttt{P}_3$};
		\node [style=mwc] (5) at (0, 2) {$\texttt{B}_1$};
		\node [style=mwc] (6) at (-1, -1) {$\texttt{P}_6$};
		\node [style=mwc] (7) at (0, -1) {$\texttt{P}_7$};
		\node [style=mwc] (8) at (1, -1) {$\texttt{P}_8$};
		\node [style=mwc] (9) at (0, -2) {$\texttt{B}_2$};
		\node [style=none] (24) at (0, 0.375) {};
		\node [style=none] (27) at (0, -0.375) {};
		\node [style=none] (28) at (1, 0.375) {};
		\node [style=none] (29) at (1, -0.375) {};
		\node [style=none] (30) at (-2, 1.375) {};
		\node [style=none] (31) at (0.375, 1.375) {};
		\node [style=none] (32) at (0.375, 0.625) {};
		\node [style=none] (33) at (-1.5, 0.625) {};
		\node [style=none] (34) at (-1.5, -0.625) {};
		\node [style=none] (35) at (-2, -1.375) {};
		\node [style=none] (36) at (1.5, -1.375) {};
		\node [style=none] (37) at (1.5, -0.625) {};
		\node [style=none] (38) at (-0.5, -1.625) {};
		\node [style=none] (39) at (1.75, -1.625) {};
		\node [style=none] (40) at (2.25, -2.375) {};
		\node [style=none] (41) at (-0.5, -2.375) {};
		\node [style=none] (42) at (0.625, 1.375) {};
		\node [style=none] (43) at (1.375, 1.375) {};
		\node [style=none] (44) at (0.625, 0.625) {};
		\node [style=none] (45) at (1.375, 0.625) {};
		\node [style=none] (46) at (-0.5, 2.375) {};
		\node [style=none] (47) at (2.25, 2.375) {};
		\node [style=none] (48) at (1.75, 1.625) {};
		\node [style=none] (49) at (-0.5, 1.625) {};
	\end{pgfonlayer}
	\begin{pgfonlayer}{edgelayer}
		\draw (6) to (0);
		\draw (0) to (7);
		\draw (0) to (8);
		\draw (1) to (6);
		\draw (1) to (7);
		\draw (1) to (8);
		\draw (0) to (3);
		\draw (0) to (2);
		\draw (1) to (3);
		\draw (1) to (4);
		\draw [style=undirected red] (5) to (4);
		\draw [style=undirected red] (5) to (3);
		\draw [style=undirected red] (5) to (2);
		\draw [style=undirected red] (6) to (9);
		\draw [style=undirected red, in=90, out=-90] (7) to (9);
		\draw [style=undirected red] (9) to (8);
		\draw [style=blackdashed] (36.center) to (37.center);
		\draw [style=blackdashed] (37.center) to (34.center);
		\draw [style=blackdashed] (34.center) to (33.center);
		\draw [style=blackdashed] (33.center) to (32.center);
		\draw [style=blackdashed] (32.center) to (31.center);
		\draw [style=blackdashed] (31.center) to (30.center);
		\draw [style=blackdashed] (30.center) to (35.center);
		\draw [style=blackdashed] (35.center) to (36.center);
		\draw [style=blackdashed] (27.center) to (29.center);
		\draw [style=blackdashed] (29.center) to (28.center);
		\draw [style=blackdashed] (28.center) to (24.center);
		\draw [style=blackdashed] (24.center) to (27.center);
		\draw [style=undirected dashed] (42.center) to (44.center);
		\draw [style=undirected dashed] (44.center) to (45.center);
		\draw [style=undirected dashed] (45.center) to (43.center);
		\draw [style=undirected dashed] (42.center) to (43.center);
		\draw [style=blackdashed] (38.center) to (41.center);
		\draw [style=blackdashed] (41.center) to (40.center);
		\draw [style=blackdashed] (40.center) to (47.center);
		\draw [style=blackdashed] (38.center) to (39.center);
		\draw [style=blackdashed] (39.center) to (48.center);
		\draw [style=blackdashed] (48.center) to (49.center);
		\draw [style=blackdashed] (49.center) to (46.center);
		\draw [style=blackdashed] (46.center) to (47.center);
	\end{pgfonlayer}
\end{tikzpicture}
    \caption{Changing the Source Node: In this case, the source node is $\texttt{P}_4$ (a professor) and we obtain a different, although intuitive partitioning:  $\texttt{P}_5$ is a colleague of $\texttt{P}_4$, $\{\texttt{P}_1, \texttt{P}_2, \texttt{P}_6, \texttt{P}_7, \texttt{P}_8 \}$ are $\texttt{P}_4$'s students, $\texttt{P}_3$ is a student that $\texttt{P}_4$ is not teaching and $\{\texttt{B}_1, \texttt{B}_2\}$ are the academic books of the department. Note how it is possible to extract these abstract concepts, even though the only explicit information we initially provided was that entities are either books or people.}
    \label{fig:changing_source_node}
\end{figure}
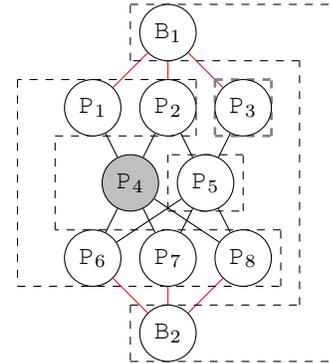

After hierarchical clustering, PRISM applies symmetry clustering in two stages: first by identifying distance-symmetric sets based on their truncated hitting times, then by identifying path-symmetric nodes within these sets based on path distributions. The first stage only serves to speed up the subsequent path-symmetric clustering since path-symmetry implies distance symmetry (Rmk. \ref{distpathsymremark}), but checking distance symmetry is quicker ($\mathcal{O}(n)$ vs $\mathcal{O}(n \ln {n})$ for \ourspaths). Note that, in this example, the source node was chosen as $\texttt{B}_1$ and the hypergraph has a high degree of symmetry relative to the source node, which explains why the distance-symmetric and path-symmetric sets are almost identical (Fig. \ref{fig:concept_finding} (ii) and (iii)). For more realistic datasets, where global symmetries in a hypergraph are rare, the differences between distance-symmetric and path-symmetric clustering will be more pronounced.

We finish this section by illustrating the effect of changing the source node of the random walks. Recall that sets of symmetric nodes, i.e., the abstract concepts, are always found with respect to a specific source node. Changing the source node, therefore, changes the learnt concepts. This idea is illustrated in Fig \ref{fig:changing_source_node}, where the source node is changed from $\texttt{B}_1$ to $\texttt{P}_4$, resulting in different clusterings. When random walks were run from $\texttt{B}_1$ we obtained the familiar concepts of teachers, colleagues, students and books. However, Fig \ref{fig:changing_source_node} illustrates how abstract concepts can often be less intuitive, but still illustrate subtle relationships in the data. In PRISM we run random walks from each node in the hypergraph in turn. This helps to identify a wide range of abstract concepts.

%%%%%%%%%%%%%%%% RESULTS %%%%%%%%%%%%%%%%%%%%%%%
\section{Experiments}
\label{sec:experiments}

We compare our motif-finding algorithm, \texttt{PRISM}, against the current state-of-the-art, \texttt{LSM} \cite{kok_learning_nodate} and \texttt{BOOSTR} \cite{boostr}, in terms of speed and accuracy of the mined MLNs. 

\textbf{Datasets} We used benchmark datasets adopted by the structure learning literature: \texttt{UW-CSE} \cite{richardson_markov_2006}, \texttt{IMDB}, and \texttt{WEBKB}. The \texttt{IMDB} dataset is subsampled from the IMDB.com database and describes relationships among movies, actors and directors. The \texttt{UW-CSE} dataset describes an academic department and the relationships between professors, students and courses. The \texttt{WEBKB} consists of Web pages and hyperlinks collected from four
computer science departments. Each dataset has five splits. 
Each time we used one split to test accuracy and the remaining splits for training. The reported results are the average over all five permutations.

\textbf{Problem} Given a dataset with partial observations, we want to predict the truth values for unobserved data. For example, for the $\texttt{IMDB}$ dataset we might not know every actor who starred in a movie. We then predict, for each actor in the database, the likelihood of an actor to have starred in a given movie. We remark that our unobserved data spans across every possible predicate in the database, e.g. for \texttt{IMDB} this would include \textsc{StarringIn}(\texttt{movie},\texttt{person}),  \textsc{Actor}(\texttt{person})$\dots$ This problem thus reduces to predicting missing edges in the hypergraph. 

\textbf{Baseline and Evaluation} We used the entire \texttt{LSM} pipeline \cite{kok_learning_nodate} as a baseline. We used the lifted belief propagation inference tool of \texttt{Alchemy} \cite{kok_alchemy_2010} to calculate the averaged conditional log-likelihood on each entity (ground atom) in the test split. For \texttt{LSM}, we used the same hyperparameters as originally adopted by the authors \cite{kok_learning_nodate}. In addition, we compared our work to the authors' publically available implementation of \texttt{BOOSTR} \cite{boostr}.

\textbf{Experiment Setup} Firstly, we ran \texttt{PRISM} and then, the remainder of the unmodified \texttt{LSM} pipeline. We used $\varepsilon = 0.1$ and $\alpha = 0.01$ throughout, as both are dataset-independent. We ran all experiments on a desktop with 32Gb of RAM and a 12-core 2.60GHz i7-10750H CPU.

\textbf{Metrics} We used standard measures from the structure learning literature. 
In particular, we measured accuracy and conditional log-likelihood for all datasets, as well as the area under the precision-recall curve (AUC) as it provides a more robust measure. The runtimes of the motif-finding step and the overall structure learning time are also reported.  

\textbf{Results} In Table \ref{tab:results_acc}, we see that compared to \texttt{LSM}, we improve in accuracy on the \texttt{IMDB} dataset by 6\%, while on \texttt{UW-CSE} and \texttt{WEBKB} the improvement is negligible. This is because LSM already found rules that generalized the data extremely well. However, as expected, the runtime of our algorithm is significantly reduced for all datasets. For motif finding, we see that our optimised algorithm is 10x-20x faster than LSM's motif-finding time. The overall structure learning computation is up to 5x faster than \texttt{LSM}. This is despite the main computational bottleneck for structure learning occurring during rule induction and evaluation - parts of the pipeline that were left unmodified. This suggests that our algorithm more tightly constrains the subsequent rule induction by finding more accurate motifs, thereby giving a speed improvement in these areas of the pipeline too.

We are slower compared to \texttt{BOOSTR}. However, \texttt{PRISM}'s accuracy drastically improves over \texttt{BOOSTR}. We believe the differences in time and accuracy between datasets for \texttt{BOOSTR} stem from the quality of the background knowledge: while on \texttt{IMDB} and \texttt{UW-CSE} background knowledge was given, on \texttt{WEBKB} it was not. No background knowledge was provided to \texttt{LSM} or \texttt{PRISM}.

%%%%%%%%%%%%%%% RELATED %%%%%%%%%%%%%%%%%%%%%%%%%
\section{Related Work}

In this section, we will review prior art in structure learning across a variety of logical languages. As we show below, every one of these approaches is based on learnt or user-defined templates to restrict the search space of candidate formulae. These templates are exactly the motifs that we are finding automatically and efficiently with the proposed framework.

\textbf{ILP} 
To alleviate the need to manually provide logical theories, several communities have developed techniques for inducing logical theories. One of the most influential family techniques for mining Horn clauses is that of \emph{Inductive Logic Programming} (ILP), e.g., FOIL \cite{FOIL}, MDIE \cite{mdie} and Inspire \cite{Inspire}.
Recently, Evans and Grefenstette proposed a differentiable variant of ILP \cite{diffILP} to support the mining of theories in noisy settings. 
ILP techniques require users to provide in advance the patterns of the formulas to mine, as well as to provide both positive and negative examples. The above requirements, along with issues regarding scalability \cite{diffILP}, restrict the application of ILP techniques in large and complex scenarios. Our work specifically focuses on finding these patterns automatically and are not restricted to Horn clauses.

Recently, several techniques aim to mine rules in a differentiable fashion. One of them is Neural LP \cite{neurallp}, a differentiable rule mining technique based on TensorLog \cite{tensorlog}. The authors in \cite{guu-etal-2015-traversing} presented a RESCAL-based model to learn from paths in knowledge graphs, while Sadeghian et al. proposed DRUM, a 
differentiable technique for learning uncertain rules in first-order logic \cite{DBLP:conf/nips/SadeghianADW19}.  
A limitation of the above line of research is that they mainly focus on rules of a specific transitive form only.
Other techniques for differentiable rule mining have been proposed in 
\cite{DBLP:journals/corr/abs-1711-05851,DBLP:journals/corr/abs-1807-08204,NIPS2017_b2ab0019}. In contrast to this line of work, our motif-finding algorithm helps in pipelines that can find more general first-order logic rules.

\textbf{MLN} The first and somewhat naive structure learning algorithm proposed for MLNs is called \emph{top-down structure learning} (TDSL) \cite{kok_learning_2005}. 
The idea is to perform a near-exhaustive search for candidate logical rules and then construct an MLN by recursively retaining rules that lead to the best improvement in the pseudo-likelihood approximation. That means that the algorithm starts with S2 directly.
Due to the exponential search space, the algorithm is not able to find long rules in large datasets and fails to find rules that truly generalize and capture the underlying data. The following approaches in this line of research all prepended the rule generation with a pattern-finding step (S1) - the core of our research. The first paper that proposed such an approach was \emph{bottom-up structure learning} (BUSL) \cite{mihalkova_bottom-up_2007}, where the idea was to pre-specify template networks akin to our motifs, that would be good candidates for potential rules and iteratively build on these templates to find more complex rules. 

To tackle the high computational overhead of structure learning, \cite{boostr} introduce BOOSTR, a technique that simultaneously learns the weights and the clauses of an MLN. The key idea is to transform the problem of learning MLNs by translating MLNs into regression trees and then uses functional gradient boosting \cite{friedman2000greedy} along those trees to find clauses. Further, it tries to learn under unobserved data. To this end, they introduced an EM-based boosting algorithm for MLNs. This approach also requires templates, however, they must be user-defined, which requires additional effort and can restrict applications. While showing promising results in terms of runtime, the technique supports only Horn clauses, and its performance drastically decreases in the absence of background knowledge, as we later show in our empirical results (Sec. \ref{sec:experiments}).

The current SOTA in this line of research is \emph{learning through structural motifs} (LSM) \cite{kok_learning_nodate}, where similar to the template-networks, motifs are identified in the hypergraph representation of the data, by running random walks on the graph and identify symmetric patterns through path signature symmetry of the random walks. The finding of good motifs or templates is the differentiating point between the different algorithms and has been shown to have the most significant impact on the quality of the ultimate rules. A principled, robust and efficient algorithm for finding such motifs could therefore improve these and future algorithms. We believe that we are the first to propose such a principled and efficient algorithm for finding motifs. 

%%%%%%%%%%%% CONCLUSION
\section{Conclusion}
We made a key step toward learning the structure of logical theories - mining structural motifs. We presented the first principled mining motif technique in which users can control the uncertainty of mined motifs and the softness of the resulting rules. Furthermore, we reduced the overall complexity of motif mining through a novel $\mathcal{O}(n\ln{n})$ clustering algorithm. Our empirical results against the state-of-the-art show improvements in runtime and accuracy by up to 80\% and 6\%, respectively, on standard benchmarks. While we focused on lifted graphical models, our work can be used to learn the formulas of other types of logical theories as well. 
One interesting direction of future work is to integrate our motif-mining technique with differential rule mining approaches as our empirical analysis shows that purely-symbolic based approaches for that task can sometimes be the bottleneck. A second direction is to integrate our motif-mining approach with Graph Neural Networks and provide a similar formal analysis. 

\bibliography{aaai23}
\clearpage
\appendix
\section{Proofs}

In this section, we provide proof of all theorems in the main paper.

Throughout this section let $\mathcal{H}=(V,E)$ be a hypergraph on which $N$ random walks of length $L$ are run from a source node $v_i \in V$. Further, let $v_j, v_k \in V \setminus \{v_i\}$ denote nodes other than the source node with $v_j \neq v_k$. In addition to the notation of Section \ref{section:preliminaries}, let $P_{i,j}\left(h\right)$ denote the probability of obtaining a truncated hitting time $h \in \{1, 2, ..., L\}$ for $v_j$ when running a \textit{single} random walk from $v_i$ and let $P_{i,j}$ denote the corresponding probability distribution.

\subsection{Proof of Theorem \ref{theorem:optimality} (Upper bound on $\varepsilon$-optimal number of random walks)}

To prove Theorem \ref{theorem:optimality} we will separately derive the $\varepsilon$-optimal number of random walks necessary for (i) accurate truncated hitting time estimates (Lemma \ref{theorem:optimalNforTHT}) and (ii) accurate path distributions (Lemma \ref{N_path_count_theorem}). Theorem \ref{theorem:optimality} then follows from taking a maximum of these two quantities. 

Note that when proving Lemma \ref{N_path_count_theorem}, the path distributions $P^{L}_{i,j}$s are hypergraph-dependent quantities. So to calculate a generic upper bound, we will need to make a distributional assumption on the $P^{L}_{i,j}$s. Specifically, we assume a Zipfian distribution i.e. an inverse rank-frequency distribution. This is reasonable since the Zipfian distribution is, to a first-order approximation, the rank-frequency distribution that arises in uncontrolled random sampling environments of complex systems (in our case, large hypergraphs) \cite{belevitch_statistical_1959}. Further, we verified empirically that Zipfian path probabilities were obtained on average when random walks were run on the experimental datasets.

\begin{lemma}
\label{theorem:optimalNforTHT}
An upper bound on the $\varepsilon$-optimal number of random walks $N$ for accurate truncated hitting time estimates on $\mathcal{H}$ is given by 
\[
N = \frac{(L-1)^2}{4\varepsilon^2}.
\]
\end{lemma}
\begin{proof}
Observe that $P_{i,j}$ is a bounded probability distribution with lower bound $1$ and upper bound $L$. Hence, from Popoviciu's inequality\cite{popoviciu_sur_1935}, its standard deviation $\sigma$ is also bounded:
\begin{equation}
 \label{popoviciu's inequality}
 \sigma \leq \frac{L-1}{2}.
\end{equation} 
Let $h_1, h_2, ..., h_N \stackrel{iid}{\sim} P_{i,j}$ be $N$ truncated first hitting times obtained by running $N$ random walks.  From the law of large numbers it follows that the $(L, N)$ truncated hitting time estimate $\hat{h}^{L,N}_{ij} \vcentcolon= \frac{1}{N}\sum^{N}_{k=1} h_k$ has standard deviation $\sigma_{N} = \sigma/\sqrt{N}$ where, using equation \eqref{popoviciu's inequality} we have
\[
\begin{split}
\label{upper_bound_mean_std}
\sigma_N \leq \frac{L-1}{2 \sqrt{N}}.
\end{split}
\]
Let $\varepsilon$ be the desired expected uncertainty in $\hat{h}^{L,N}_{ij}$ as per Definition \ref{definition:uncertainty}. We have
\begin{equation}
    \varepsilon \vcentcolon= \frac{\sigma_N}{h^{L,N}_{ij}} \leq \sigma_N \leq \frac{L-1}{2 \sqrt{N}},
\end{equation}
which upon rearranging implies that the $\varepsilon$-optimal $N$ for accurate truncated hitting times satisfies
\begin{equation}
\label{N_equation}
        N \leq \frac{(L-1)^2}{4\varepsilon^{2}}.
\end{equation}
Taking the upper bound of (\ref{N_equation}) completes the proof.
\end{proof}

\begin{lemma}
\label{N_path_count_theorem}
An upper bound on the $\varepsilon$-optimal number of random walks $N$ for accurate path probabilities on $\mathcal{H}$ is given by 
\[
N = \frac{P^{*}\left(\gamma + \ln P^{*}\right)}{\varepsilon^2}
\]
where ${P^{*} = 1 + {e\left(e^{L}-1\right)}/({e-1}) \gg 1}$, $e$ is the number of unique edge labels in $\mathcal{H}$, and $\gamma \approx 0.577$ is the Euler-Mascheroni constant. 
\end{lemma}

\begin{proof}

Let $P^{L}_{i,j}(k)$ denote the path-probability of the $k^{th}$ most common path signature of node $v_j$. The path-probabilities are assumed Zipfian distributed, that is
\begin{equation}
    P^{L}_{i,j}(k) = \frac{1}{kZ},
\end{equation}
where $Z = \sum_{k=1}^{P} \frac{1}{k}$ is the normalisation constant and $P \vcentcolon= \vert \mathcal{S}^{L}_{i,j}\vert$ is the number of distinct path-signatures in the path distribution. Further, since each walk is an independent sample from this path-probability distribution, we have that the path counts $\hat{C}^{L,N}_{i,j}(k)$ are binomially distributed with $\E\left[\hat{C}^{L,N}_{i,j}(k)\right] = NP^{L}_{i,j}(k)$ and $\Var\left[\hat{C}^{L,N}_{i,j}(k)\right] = N P^{L}_{i,j}(k)(1-P^{L}_{i,j}(k))$. It follows that the $(L,N)$ path-probability estimate $\hat{P}^{L,N}_{i,j}(k) \vcentcolon= \hat{C}^{L,N}_{i,j}(k) / N$ has fractional uncertainty $\epsilon(k)$ given by

\begin{equation}
\begin{split}\label{epsilon_k}
    \epsilon\left(k\right) \vcentcolon &= \frac{\sqrt{\Var\left[\hat{P}^{L,N}_{i,j}(k)\right]}}{\E\left[\hat{P}^{L,N}_{i,j}(k)\right]}\\ &= \sqrt{\frac{1-  P^{L}_{i,j}(k)}{N P^{L}_{i,j}(k)}}\\ &= \sqrt{\frac{kZ - 1}{N}}\\ &= \sqrt{\frac{k\left(\sum_{m=1}^{P}\frac{1}{m}\right) - 1}{N}}.
\end{split}
\end{equation}

Let $\varepsilon$ denote the desired expected uncertainty in $\hat{P}^{L,N}_{i,j}$ as per Definition \ref{definition:uncertainty}, then we identify 
\[\varepsilon = \max_{k \in \{1, 2, \dots, P\}} \epsilon(k) = \epsilon(P)\]
and so, upon rearranging, 
\begin{equation}
N = \frac{P\left(\sum_{m=1}^{P}\frac{1}{m}\right) -1}{\varepsilon^2}
\end{equation}
is the $\varepsilon$-optimal $N$ for obtaining accurate path distributions.

Unfortunately, computing $P = \vert \mathcal{S}^{L}_{i,j}\vert$ directly is challenging since it depends on the hypergraph geometry and is different for different source nodes. However, an upper bound $P^* \geq P$ can be easily constructed as follows:

Let $e$ be the number of unique edge labels of the hypergraph. From the definition of a path signature being an ordered sequence of hyperedge labels, the maximum number of unique paths of length $k$ between two nodes is, therefore, $e^k$. To compute $P^*$ we sum over all possible path lengths $k \in \{1, 2, ..., L\}$ and so

\begin{equation}
\label{N_paths}
  P^* = 1 + \sum_{k=1}^{L} e^{k} = 1 +  \frac{e(e^{L}-1)}{e-1}, 
\end{equation}
where the contribution of $+1$ stems from also considering the \textit{null path}, arising whenever a node is not hit during a random walk. Since $P^* \geq P$ we have that the $\varepsilon$-optimal $N$ satisfies 
\begin{equation}
\label{N_upper_bound}
N \leq \frac{P^{*}\left(\sum_{m=1}^{P^{*}}\frac{1}{m}\right) -1}{\varepsilon^2} \approx \frac{P^{*}\left(\gamma + \ln P^*\right)}{\varepsilon^2}, 
\end{equation}

\noindent
where we have used the assumption $P^* \gg 1$ to make the log-integral approximation for the sum of harmonic numbers $\sum_{m=1}^{P^{*}} \frac{1}{m} = \gamma + \ln P^* + \mathcal{O}\left(\frac{1}{P^*}\right)$ where $\gamma \approx 0.577$ is the Euler-Mascheroni constant\footnote{To see why $P^* \gg 1$ is a reasonable assumption, consider a typical real-world hypergraph such as that based on the IMDB dataset. This hypergraph has only ten predicates ($e=10$). Taking even a modest random walk length of $L=3$ gives $P^* = 1 + e(e^L -1)/(e-1) \approx 1100 \gg 1$.}. Taking the upper bound of (\ref{N_upper_bound}) completes the proof.
\end{proof}

\subsection{Proof of Theorem \ref{theta_sym_hypothesis} (Distance-Symmetric Hypothesis Test)}

\begin{proof}
We reject the null hypothesis that two nodes are distance symmetric if a two-sample t-test\cite{schiefer_statistics_2021} on the nodes' truncated hitting times fails. The relevant t-statistic is
\begin{equation}
  t = \frac{ \hat{h}^{L,N}_{i,j} - \hat{h}^{L,N}_{i,k} }{ \frac{\hat{\sigma}}{\sqrt{N}}}, 
\end{equation}
\noindent
where $\hat{\sigma} = \sqrt{\hat{\sigma}_j^2 + \hat{\sigma}_k^2}$ and $\hat{\sigma}_j, \hat{\sigma}_k$ are unbiased estimators of the standard deviation of the truncated hitting times of the two nodes. In a two-tailed test, the null hypothesis is rejected at significance level $\alpha$ if $|t| > t_{\alpha/2, N-1}$. Further, from Popoviciu's inequality\cite{popoviciu_sur_1935} we know that  $\hat{\sigma}_i, \hat{\sigma}_j \leq \frac{L-1}{2}$ (c.f. proof of Lemma \ref{theorem:optimalNforTHT}) hence $\hat{\sigma} \leq \frac{L-1}{\sqrt{2}}$. Together this implies 
\[
\frac{\sqrt{2N}}{L-1}\vert \hat{h}^{L,N}_{i,j} - \hat{h}^{L,N}_{i,k} \vert \geq \vert t \vert > t_{\alpha/2, N-1}.
\]
Rearranging gives the required inequality on $\vert \hat{h}^{L,N}_{i,j} - \hat{h}^{L,N}_{i,k} \vert$.

\end{proof}

\subsection{Proof of Theorem \ref{hypothesis_test} (Path-Symmetric Hypothesis Test)}

\begin{proof}

We begin by assuming that $H_0$ is true, i.e. the nodes in $B_k$ are order-$L$ path-symmetric and so have identical path-probability distributions. We now consider only the subset of the paths that have fixed length $l \leq L$ and denote the probability of obtaining the $\lambda^{th}$-indexed path signature during a random walk as $\pi_\lambda \vcentcolon= \lim_{N \xrightarrow{} \infty }\hat{C}^{L,N}_{i,j}\vert_l(\lambda) / N$, where $\lambda \in \left\{1, 2, ..., \Lambda\right\}$. Also define $\pi_0$ such that $\sum_{i=0}^{\Lambda}\pi_i = 1$, where $\pi_0$ can be interpreted as the probability that a node is not reached after exactly $l$ steps from the source node. Each random walk is an independent event so the marginal path counts are samples from an $N$-trial multinomial distribution with parameters $\left(\pi_0, \pi_1, ..., \pi_{\Lambda}\right)$. Hence, from standard properties of the multinomial distribution, we have $\E\left[c_\lambda^{(k)}\right]=N\pi_\lambda$, 
$\Var\left[c_\lambda^{(k)}\right]=N\pi_\lambda\left(1-\pi_\lambda\right)$ and $\Cov\left[c^{(k)}_\lambda, c^{(k)}_{\lambda'}\right] = -N\pi_\lambda\pi_{\lambda'}$. 

Assuming $N \pi_\lambda \gg 1$ for all paths, we make the multivariate Gaussian approximation to the multinomial distribution\cite{johnson_approximation_1960} and approximate the path-count vector of $v_j$, denoted $\mathbf{c}^{(j)} = (c^{(j)}_0, c^{(j)}_1, ..., c^{(j)}_{\Lambda})$, to be multivariate Gaussian distributed as
\[
\mathbf{c}^{(j)} \stackrel{iid}{\sim} \N(\boldsymbol{\pi}, \Sigma)
\]
where $\boldsymbol{\pi} \vcentcolon= (\pi_0, \pi_1, ..., \pi_{\Lambda})$ and 
\begin{equation}
\Sigma_{\lambda,\lambda'} \vcentcolon= \begin{cases}
    = N\pi_\lambda(1-\pi_\lambda)  &\text{for} \; \lambda=\lambda', \\[10pt]
    = -N\pi_\lambda\pi_{\lambda'} &\text{for} \; \lambda\neq \lambda'.
    \end{cases}
\label{eqn:count_covariance}
\end{equation}

\noindent
Next, consider the random vector 

\[\mathbf{X} = \left(\mathbf{X}^{(1)}, \mathbf{X}^{(2)}, ..., \mathbf{X}^{(\vert B_k \vert)}\right),\]

\noindent
with components defined as 
\[
X^{(j)}_\lambda \vcentcolon= c_\lambda - c^{(j)}_\lambda.
\]
Under the same assumption, $\mathbf{X}$ also multivariate Gaussian distributed: \[
\mathbf{X} \sim \N(\mathbf{0}, \boldsymbol{\Tilde{\Sigma}})
\]
where $\boldsymbol{\Tilde{\Sigma}}$ is a block matrix with components
\[
\Tilde{\Sigma}^{(b, b')}_{\lambda, \lambda'} \vcentcolon= \Cov\left[X^{(b)}_\lambda, X^{(b')}_{\lambda'}\right],
\]
where upper indices index blocks and lower indices index within blocks.
To evaluate $\boldsymbol{\Tilde{\Sigma}}$ we substitute in the definitions of $X_{\lambda}^{(b)}$ and $c_\lambda$ and using $\Cov\left[c_\lambda^{(b)},c_{\lambda'}^{(b')}\right] = \delta_{b,b'} \Sigma_{\lambda, \lambda'}$ to obtain
\begin{equation}
    \Tilde{\Sigma}^{(b,b')}_{\lambda, \lambda'} = \left(\delta_{b,b'}- \frac{1}{\vert B_k \vert}\right)\Sigma_{\lambda, \lambda'},
\label{eqn:covariance_X}
\end{equation}

\noindent
where $\Sigma_{\lambda, \lambda'}$ is given by equation \eqref{eqn:count_covariance}. 
We now construct the test statistic $Q(B_k)$, defined as
\[
Q(B_k) \vcentcolon= \sum_{\lambda=0}^{\Lambda}\sum_{v_j \in B_k} \left(X^{(j)}_\lambda\right)^2 = \sum_{\lambda=0}^{\Lambda}\sum_{v_j \in B_k} \left(c_\lambda - c^{(j)}_\lambda\right)^2
\]
which is the sum of squares of zero-mean, correlated normal random variables. 

A standard result from classical statistics states that the sums of squared, correlated, normal random variables is identical in distribution to a weighted sum of squares of independent, standard normal random variables. In particular, we have that
\[
Q(B_k) \sim \sum_{\lambda=0}^{\Lambda}\sum_{v_j \in B_k} w^{(j)}_\lambda \left(Z^{(j)}_\lambda\right)^2,
\]
where
\[
Z^{(j)}_\lambda \stackrel{iid}{\sim} \N(0,1)
\]
and $w^{(j)}_\lambda$, for $\lambda \in (0, 1, 2, ..., \Lambda)$, $j \in (1, 2, ..., \vert B_k \vert)$, are the eigenvalues of $\boldsymbol{\Tilde{\Sigma}}$\cite{mathai_quadratic_1992}. It follows from standard definitions that $Q(B_k)$ is distributed as a central, generalised chi-squared distribution with weight parameters $\{w_{\lambda}^{(j)}\}$ and corresponding degrees of freedom vector $\boldsymbol{\nu} = (1, 1, ..., 1)$\cite{moschopoulos_distribution_1984}. For shorthand we denote this as $Q(B_k) \sim \chi^2_{\boldsymbol{w}, \boldsymbol{\nu}}$.

Finally, note that if $N \gg 1$, as assumed in the theorem, then to $\mathcal{O}(1/N)$ accuracy, we can replace the $\pi_\lambda$ in equation \eqref{eqn:covariance_X} with their unbiased estimators $\frac{c_\lambda}{N}$. We thus conclude that the test statistic $Q(B_k)$ is approximately distributed as $Q(B_k) \sim \chi^2_{\boldsymbol{w}, \boldsymbol{\nu}}$ with $\boldsymbol{w}$ and $\boldsymbol{\nu}$ as defined in the theorem statement. We observe that the null hypothesis is violated at significance level $\alpha$ if $Q(B_k) > \chi^2_{\boldsymbol{w}, \boldsymbol{\nu}}(\alpha)$, as required. 
\end{proof}

\subsection{Proof of Theorem \ref{theorem:fewer_random_walks} (Fewer Random Walks)}

\begin{proof}
This follows immediately from (\ref{epsilon_k}) by applying $\varepsilon \geq \epsilon(k+1, P^*)$ and taking the upper bound on $N$. The reason that $k+1$ appears rather than $k$ is to account for the possibility that the null path is more probable than the $k^{th}$ most likely non-null path.
\end{proof}

\section{Implementation Details}

\subsection{A Gamma Distribution Approximation for a Generalised Chi-Squared Distribution}
\label{sec:gamma_distribution_approximation}

\textbf{Motivation} Implementing the hypothesis test of Theorem \ref{hypothesis_test} requires computing critical values of generalised chi-squared distributions. Unfortunately, few numerical packages implement generalised chi-squared distributions, and those that do rely on numerical integration each time a critical value is computed. We deemed this to be too computationally costly for our application. To reduce computation and simplify the implementation we instead approximate the generalised chi-squared distribution with a gamma distribution\footnote{Approximating skew distributions by gamma mixture models is a standard approach known as the Satterthwaite-Welch method \cite{welch_significance_1938,satterthwaite_approximate_1946}. In our work, we were able to obtain a satisfactory approximation using just a single gamma distribution.}. 

\textbf{Method} A generalised chi-squared distribution $\chi^2_{\boldsymbol{w}, \boldsymbol{\nu}}$ has mean $\mu = \sum_i w_i \nu_i$ and variance $\sigma^2 = 2 \sum_i w_i^2 \nu_i$\cite{davies_numerical_1973} where, in the context of Theorem \ref{hypothesis_test}, $\boldsymbol{w}$ are the eigenvalues of the block matrix $\Tilde{\Sigma}^{(b, b')}_{\lambda, \lambda'}$ and $\boldsymbol{\nu}$ is a vector of ones. We shall approximate $\chi^2_{\boldsymbol{w}, \boldsymbol{\nu}}$ by a gamma distribution with the same mean and variance. This approximation incurs a significant computational advantage since the distribution parameters can be computed without performing the eigenvalue decomposition that is implied in the statement of Theorem \ref{hypothesis_test}. To see why, notice that in computing $\mu$ and $\sigma^2$ it suffices to know only the sum and sum of squares of the eigenvalues of $\Tilde{\Sigma}^{(b, b')}_{\lambda, \lambda'}$. Standard linear algebra manipulations then imply
\begin{equation}
\label{eqn:gammamean}
\mu = \sum_{\lambda=0}^{\Lambda} \sum_{j = 1}^{\vert B_k \vert} w_\lambda^{(j)} = Tr\boldsymbol{\Tilde{\Sigma}} = (\vert B_k \vert - 1)  \sum_{\lambda=0}^{\Lambda} \Sigma_{\lambda, \lambda} 
\end{equation}
and 
\begin{equation}
\begin{split}\label{eqn:gammavariance}
\sigma^2 &= 2\sum_{\lambda=0}^{\Lambda} \sum_{j = 1}^{\vert B_k \vert} \left(w_\lambda^{(j)}\right)^2\\ &= 2 Tr \left( \boldsymbol{\Tilde{\Sigma}}^2 \right)\\ &= 2 \left(\vert B_k \vert - 1\right) \sum_{\lambda,\lambda'=0}^{\Lambda} \Sigma^2_{\lambda, \lambda'},
\end{split}
\end{equation}
where the matrix components $\Sigma_{\lambda, \lambda'}$ are defined in \eqref{eqn:count_covariance}. Thus $\mu$ and $\sigma^2$ can be calculated without knowledge of the eigenvalues of $\Tilde{\Sigma}^{(b, b')}_{\lambda, \lambda'}$. The gamma distribution having the same mean and variance is then $Gamma(\alpha, \beta)$, with rate parameter $\beta = \frac{\mu}{\sigma^2}$ and scale parameter $\alpha = \frac{\mu^2}{\sigma^2}$. For the databases we tested on, we found this to be an excellent approximation, with the gamma-estimated critical values being at most $1\%$ different from the exact values.

\subsection{The Computational Complexity of \space \ourspaths}

From the pseudocode in Section \ref{section:ours:path-symmetry}, we see that the complexity of $\ourspaths$ depends on the computational complexity of the unsupervised clustering algorithm and the hypothesis test. Here we evaluate the computational complexity of these component stages and thus the overall algorithm. In our calculations, we assume that we are clustering $n$ nodes amongst which there are $\Lambda_l$ unique path signatures of length $l \leq L$, and let $\Lambda \vcentcolon= \sum_{l=1}^{L} \Lambda_l$.

\noindent
\textbf{Complexity of Unsupervised Clustering} For illustration we consider two clustering algorithms: BIRCH \cite{zhang_birch_1996}, which has $\mathcal{O}(n)$ complexity, and 2-means (i.e. the $k$-means algorithm \cite{forgy_cluster_1965} with $k=2$), which has $\mathcal{O}(n^2)$ complexity.

\noindent
\textbf{Complexity of the Hypothesis Test} Recall that we use the gamma distribution approximation to the generalised chi-squared distribution (see Section \ref{sec:gamma_distribution_approximation}). Examining equations \eqref{eqn:gammamean} and \eqref{eqn:gammavariance} we see that one application of the hypothesis test for exact order-$l$ path-symmetry is $\mathcal{O}(\Lambda_l^2)$ complexity. When testing whether a given set of nodes is path-symmetric, in the best case we have to perform the order-$l$ exact hypothesis test just once (considering only paths of length $l=L$, and then failing the test). In the worst case we have to apply it $L$ times (considering paths separately of all lengths $L, L-1, \dots, 2, 1$, with the test passing at each stage).

\noindent
\textbf{Overall Complexity} The clustering in $\ourspaths$ is hierarchical using binary cuts, so in the best case we require $\mathcal{O}(\ln_2{n})$ clusterings and hypothesis tests (in the worst case we require $\mathcal{O}(n)$). It follows from standard computation that the best-case complexity of $\ourspaths$ is $\mathcal{O}(\ln_2{n}\left(n + \Lambda^2\right))$ when using BIRCH clustering and $\mathcal{O}(n^2 + \Lambda^2\ln_2{n})$ when using 2-means.

\subsection{Hierarchical Clustering Algorithm}

Here we present pseudocode for our $\hcluster$ algorithm. The algorithm is discussed in overview in Section \ref{section:h-clustering}.

\begin{algorithm}[!htbp]
\caption{$\hcluster$}
\label{hierarchical_clustering_algorithm}
\textbf{Input:}{$\mathcal{H}=(V,E)$, a hypergraph}\\
\textbf{Parameters}{$\lambda^{max}_2$, the maximum value for the second smallest eigenvalue of the symmetric Laplacian matrix of the graph clusters $n_{min}$, the minimum graph cluster size}\\
\textbf{Output:}{$H=\{\mathcal{H}_1, \mathcal{H}_2, ..., \mathcal{H}_k\}$, a set of hypergraphs with vertex sets $V_1, V_2, ..., V_k$ and edge sets $E_1, E_2, ..., E_k$ satisfying $V = V_1 \bigcup V_2 \bigcup ... \bigcup V_k$ and $E = E_1 \bigcup E_2 \bigcup ... \bigcup E_k$.}

$\mathcal{G} \xleftarrow{} \texttt{ConvertToGraph}(\mathcal{H})$ \tcp*{Sec.~\ref{section:h-clustering}} 

$G \xleftarrow{} \texttt{GetClusters}(\mathcal{G})$ \tcp*{Alg.~\ref{get_clusters}}

$H \xleftarrow{} \emptyset$ \\
\For{$\mathcal{G}_i \in G$}{
$\mathcal{H}_i \xleftarrow{} \texttt{ConvertToHypergraph}(\mathcal{G}_i)$ \tcp*{Sec.~\ref{section:h-clustering}}
Add $\mathcal{H}_i$ to $H$
}
\Return{$H$}
\end{algorithm}

\begin{algorithm}[!htbp]
\caption{GetClusters}
\label{get_clusters}
\textbf{Input:}{$\mathcal{G}$, a graph}\\
\textbf{Parameters:}{$\lambda^{max}_2$, $n_{min}$}

Compute $\lambda_2$ and $\mathbf{v}_2$, the second smallest eigenvalue and eigenvector of  $\mathcal{G}$'s symmetrical Laplacian matrix $L^{sym}$. \\

\uIf{$\lambda_2 > \lambda^{max}_2$}{
\Return{$\mathcal{G}$}}
\Else{
$\mathcal{G}_1, \mathcal{G}_2 \xleftarrow{} \texttt{CheegerCut}(\mathcal{G}, \mathbf{v}_2)$ \tcp*{See e.g. \cite{chang_cheegers_2017}} 

\uIf{either $\mathcal{G}_1$ or $\mathcal{G}_2$ has fewer nodes than $n_{min}$}{
\Return{$\mathcal{G}$}} 
\Else{
\Return{\texttt{GetClusters}($\mathcal{G}_1$), \texttt{GetClusters}($\mathcal{G}_2$)}}}
\end{algorithm}

\end{document}